\documentclass{article}

\usepackage[utf8]{inputenc} 
\usepackage[T1]{fontenc}    
\usepackage{url}            
\usepackage{booktabs}       
\usepackage{amsfonts}       
\usepackage{nicefrac}       
\usepackage{xcolor}         

\usepackage[top=2cm, bottom=2cm, left=2.2cm, right=2.2cm]{geometry}
\usepackage{amsthm,amsmath,amsfonts,amssymb}

\usepackage[numbers]{natbib}
\usepackage[colorlinks,citecolor=blue,urlcolor=blue]{hyperref}
    \hypersetup{
    	colorlinks = true,
    	linkcolor = blue,
    	anchorcolor = blue,
    	citecolor = blue,
    	filecolor = blue,
    	urlcolor = blue
    }

\usepackage[linesnumbered,lined,boxed,commentsnumbered,ruled,longend]{algorithm2e}

\definecolor{darkslategray}{rgb}{0.18, 0.31, 0.31}
\usepackage{wrapfig}

\usepackage{subcaption}

\definecolor{darkcyan}{rgb}{0.0, 0.55, 0.55}
\definecolor{MidnightBlue}{RGB}{25,25,112}
\definecolor{MidnightBlueComplementingGreen}{RGB}{25,112,25}
\definecolor{MidnightBlueComplementingPurple}{RGB}{112,25,112}
\definecolor{MidnightBlueComplementingRed}{RGB}{112,25,69}
\definecolor{WowColor}{rgb}{.75,0,.75}
\definecolor{MildlyAlarming}{rgb}{0.85,0.25,0.1}
\definecolor{SubtleColor}{rgb}{0,0,.50}
\definecolor{antiquefuchsia}{rgb}{0.57, 0.36, 0.51}
\definecolor{fashionfuchsia}{rgb}{0.96, 0.0, 0.63}
\definecolor{jade}{rgb}{0.0, 0.66, 0.42}
\definecolor{caribbeangreen}{rgb}{0.0, 0.8, 0.6}
\definecolor{aquamarine}{rgb}{0.5, 0.8, 0.85}
\definecolor{lightseagreen}{rgb}{0.13, 0.7, 0.67}
\definecolor{darkgreen}{rgb}{0.0, 0.2, 0.13}
\definecolor{darkspringgreen}{rgb}{0.09, 0.45, 0.27}
\definecolor{attentioncolor}{RGB}{152,90,81}
\definecolor{burgred}{RGB}{40,3,22}
\definecolor{AnnieGreen}{RGB}{17,123,92}
\definecolor{Turquoise}{RGB}{64,224,208}

\definecolor{darkjade}{RGB}{0,122,84}
\definecolor{Window1}{RGB}{92,150,31}%
    \definecolor{Window1dark}{RGB}{41,67,13}%
\definecolor{Window2}{RGB}{255,168,28}
    \definecolor{Window2dark}{RGB}{114,75,12}
\definecolor{Window3}{RGB}{255,96,33}
    \definecolor{Window3dark}{RGB}{97,36,12}
\definecolor{InputColor}{RGB}{20,255,177}
    \definecolor{InputColorlight}{RGB}{222,237,229}

\definecolor{RedAlizarin}{rgb}{0.82, 0.1, 0.26}

\usepackage{mathtools}
\usepackage{hyperref}
\makeatletter
\newcommand{\mytag}[2]{%
  \text{#1}%
  \@bsphack
  \begingroup
    \@onelevel@sanitize\@currentlabelname
    \edef\@currentlabelname{%
      \expandafter\strip@period\@currentlabelname\relax.\relax\@@@%
    }%
    \protected@write\@auxout{}{%
      \string\newlabel{#2}{%
        {#1}%
        {\thepage}%
        {\@currentlabelname}%
        {\@currentHref}{}%
      }%
    }%
  \endgroup
  \@esphack
}
\makeatother

\usepackage{soul}

\usepackage[colorinlistoftodos]{todonotes}

\NewDocumentCommand{\AK}{mo}{
    \IfValueF{#2}{
                        {{
                            \textcolor{magenta}{ 
                            \textbf{AK:}
                            \textit{{#1}}
                            }
                        }}
        }
    \IfValueT{#2}{
                        \marginnote{{\scriptsize
                            \textcolor{magenta}{ 
                            \textbf{AK:}
                            \textit{{#1}}
                            }
                        }}
        }
                    }

\NewDocumentCommand{\TF}{mo}{
    \IfValueF{#2}{
                        {{
                            \textcolor{blue}{ 
                            \textbf{Takashi:}
                            \textit{{#1}}
                            }
                        }}
        }
    \IfValueT{#2}{
                        \marginnote{{\scriptsize
                            \textcolor{blue}{ 
                            \textbf{Takashi:}
                            \textit{{#1}}
                            }
                        }}
        }
                    }

\NewDocumentCommand{\ocariz}{mo}{
    \IfValueF{#2}{
                        {{
                            \textcolor{purple}{ 
                            \textbf{Ocariz:}
                            \textit{{#1}}
                            }
                        }}
        }
    \IfValueT{#2}{
                        \marginnote{{\scriptsize
                            \textcolor{purple}{ 
                            \textbf{Ocariz:}
                            \textit{{#1}}
                            }
                        }}
        }
                    }

\NewDocumentCommand{\Bum}{mo}{
    \IfValueF{#2}{\textcolor{teal}{\textbf{Bum: }\textit{#1}}}
    \IfValueT{#2}{\marginnote{{\scriptsize\textcolor{teal}{\textbf{Bum: }\textit{#1}}}}}
}

\definecolor{cobalt}{rgb}{0.0, 0.28, 0.67}

\usepackage{graphicx}

\usepackage[T1]{fontenc}    
\usepackage{booktabs}       

\usepackage{amsfonts,amsmath,amssymb,bbm}
\usepackage{enumitem}
\usepackage{graphicx}
\usepackage{adjustbox}
\usepackage[UKenglish]{isodate}
\usepackage{graphicx}
\usepackage[font=small,labelfont=it]{caption}
\usepackage{subcaption}
\usepackage{hyperref} 
    \hypersetup{
    	colorlinks = true,
    	linkcolor = blue,
    	anchorcolor = blue,
    	citecolor = blue,
    	filecolor = blue,
    	urlcolor = blue
    }
\usepackage{cleveref}
\usepackage{float}
\usepackage{enumitem}
\usepackage{multirow}
\usepackage{fancyvrb}
\usepackage{rotating}
    \usepackage{tikz}
    \usepackage{tikz-cd} 
    \pgfdeclarelayer{edgelayer}
    \pgfdeclarelayer{nodelayer}
    \pgfsetlayers{edgelayer,nodelayer,main}
    \tikzstyle{new style 0}=[fill={rgb,255: red,255; green,94; blue,247}, draw=black, shape=circle]
    \tikzstyle{pointy}=[fill=white, draw=black, shape=circle]
    \tikzstyle{pointy}=[->]
\usepackage{fontawesome}

\usepackage{lscape}

\renewcommand{\phi}{\varphi}

\newcommand{\eqdef}{\ensuremath{\stackrel{\mbox{\upshape\tiny def.}}{=}}}

\NewDocumentCommand{\luca}{mo}{
    \IfValueF{#2}{
                        {{\scriptsize
                            \textcolor{green}{ 
                            \textbf{L:}
                            \textit{{#1}}
                            }
                        }}
        }
    \IfValueT{#2}{
                        \marginnote{{\scriptsize
                            \textcolor{green}{ 
                            \textbf{L:}
                            \textit{{#1}}
                            }
                        }}
        }
                    }
\NewDocumentCommand{\giulia}{mo}{
    \IfValueF{#2}{
                        {{\scriptsize
                            \textcolor{red}{ 
                            \textbf{GL:}
                            \textit{{#1}}
                            }
                        }}
        }
    \IfValueT{#2}{
                        \marginnote{{\scriptsize
                            \textcolor{red}{ 
                            \textbf{GL:}
                            \textit{{#1}}
                            }
                        }}
        }
}
\NewDocumentCommand{\anastasis}{mo}{
    \IfValueF{#2}{
                        {{\scriptsize
                            \textcolor{violet}{ 
                            \textbf{A:}
                            \textit{{#1}}
                            }
                        }}
        }
    \IfValueT{#2}{
                        \marginnote{{\scriptsize
                            \textcolor{violet}{ 
                            \textbf{A:}
                            \textit{{#1}}
                            }
                        }}
        }
                    }
                    
\NewDocumentCommand{\cody}{mo}{
    \IfValueF{#2}{
                        {{\scriptsize
                            \textcolor{orange}{ 
                            \textbf{A:}
                            \textit{{#1}}
                            }
                        }}
        }
    \IfValueT{#2}{
                        \marginnote{{\scriptsize
                            \textcolor{orange}{ 
                            \textbf{A:}
                            \textit{{#1}}
                            }
                        }}
        }
                    }

\NewDocumentCommand{\Greg}{mo}{
    \IfValueF{#2}{
                        {{\scriptsize
                            \textcolor{cyan}{ 
                            \textbf{Y:}
                            \textit{{#1}}
                            }
                        }}
        }
    \IfValueT{#2}{
                        \marginnote{{\scriptsize
                            \textcolor{cyan}{ 
                            \textbf{Y:}
                            \textit{{#1}}
                            }
                        }}
        }
                    }

\usepackage{appendix}
\usepackage{comment}

\NewDocumentCommand{\NN}{oo}{
    \ensuremath{
        \mathcal{NN}
        \IfValueT{#1}{_{#1}}\IfValueF{#1}{_{[d]}}
        \IfValueT{#2}{^{#2}}\IfValueF{#2}{^{\sigma}}
    }
}

\usepackage{cleveref}
\newcounter{termcounter}
\renewcommand{\thetermcounter}{\Roman{termcounter}}
\crefname{term}{term}{terms}
\creflabelformat{term}{#2\textup{(#1)}#3}

\makeatletter
\def\term{\@ifnextchar[\term@optarg\term@noarg}
\def\term@optarg[#1]#2{%
  \textup{#1}%
  \def\@currentlabel{#1}%
  \def\cref@currentlabel{[][2147483647][]#1}%
  \cref@label[term]{#2}}
\def\term@noarg#1{%
  \refstepcounter{termcounter}%
  \textup{(\thetermcounter)}%
  \cref@label[term]{#1}}
\makeatother

\NewDocumentCommand{\comp}{o}{\operatorname{Comp}
    {\IfValueT{#1}{({#1}})}
}
\NewDocumentCommand{\DAG}{o}{
\operatorname{DAG}\IfValueT{#1}{
        _{{#1}}
    }
    \IfValueF{#1}{_{d,D}}
}

\NewDocumentCommand{\Rep}{o}{
{f}
    {\IfValueT{#1}{
        _{{#1}}
    }}
    {\IfValueF{#1}{
        (_{\mathcal{C}})
    }}
}

\usepackage{amsthm}

\theoremstyle{plain}
\newtheorem{theorem}{Theorem}[section]
\newtheorem{proposition}[theorem]{Proposition}
\newtheorem{lemma}[theorem]{Lemma}
\newtheorem{corollary}[theorem]{Corollary}
\theoremstyle{definition}
\newtheorem{definition}[theorem]{Definition}
\newtheorem{assumption}[theorem]{Assumption}
\theoremstyle{remark}
\newtheorem{remark}[theorem]{Remark}

\newtheorem{example}{Example}

\newcommand{\reals}{\mathbb{R}}

\newcommand{\calF}{\mathcal{F}}

\newcommand{\calC}{\mathcal{C}}

\newcommand{\mSet}[1]{\big\{#1\big\}}

\DeclareMathOperator{\pdim}{P\text{-}\dim}

\newcommand{\deriv}[2]{\frac{\partial #1}{\partial #2}}

\newcommand{\fat}{\mathrm{fat}}

\usepackage{amsmath}
\usepackage{amssymb}
\usepackage{amsfonts}
\usepackage{mathtools}
\usepackage{amsthm}
\usepackage{algpseudocode}

\usepackage{natbib}

\theoremstyle{plain}
	\newtheorem{thm}{Theorem}
	\numberwithin{thm}{section}
	\newtheorem*{thm*}{Theorem}
	
	\newtheorem*{cor*}{Corollary}
	
	\newtheorem*{prop*}{Proposition}
	
	\newtheorem*{fact*}{Fact}
	
	\newtheorem*{lem*}{Lemma}
	
	\newtheorem*{ex*}{Exercise}
	
	\newtheorem*{claim*}{Claim}
	
	\newtheorem*{conj*}{Conjecture}
	
	\newtheorem*{question*}{Question}
	\newtheorem*{notation*}{Notation}
\theoremstyle{definition}
	\newtheorem{Def}[thm]{Definition}
	\newtheorem*{Def*}{Definition}
	
	\newtheorem*{rmk*}{Remark}
	
	\newtheorem{soln*}{Solution}
	
	\newtheorem*{note*}{Note}
	
	\newtheorem*{eg*}{Example}	
	
	\newtheorem*{construction*}{Construction}
	
	\newtheorem*{warning*}{Warning}
	
	\newtheorem*{obs*}{Observation}	
	
	\newtheorem*{recall*}{Recollection}

\numberwithin{equation}{section}

\usepackage[most]{tcolorbox}

\newtcolorbox{questionbox}{
  enhanced,
  hbox,
  colback=gray!6,
  colframe=gray!50,
  boxrule=0.4pt,
  arc=1.5mm,
  left=3mm,
  right=3mm,
  top=1.2mm,
  bottom=1.2mm,
  before=\begin{center},
  after=\end{center}
}

\definecolor{slighlywarmblack}{rgb}{0.0, 0.25, 0.25}
\definecolor{slighlylesswarmblack}{rgb}{0.0, 0.15, 0.15}
\let\oldproof\proof
\renewcommand{\proof}{\color{slighlylesswarmblack}\oldproof}

\begin{document}

\title{Every Feedforward Neural Network Definable in an o-Minimal Structure Has Finite Sample Complexity}

\author{%
  Anastasis Kratsios\thanks{Department of Mathematics \& Statistics, McMaster University, Hamilton, ON, Canada; \texttt{kratsioa@mcmaster.ca}.}
  \and
  Gregory Cousins\thanks{Department of Mathematics \& Statistics, McMaster University, Hamilton, ON, Canada; \texttt{gcousins@alumni.nd.edu}.}
  \and
  Haitz S\'aez de Oc\'ariz Borde\thanks{University of Cambridge, United Kingdom; \texttt{chri6704@ox.ac.uk}.}
  \qquad\qquad
  \and 
  Bum Jun Kim\thanks{Graduate School of Engineering, The University of Tokyo, Tokyo, Japan; \texttt{bumjun.kim@weblab.t.u-tokyo.ac.jp}.}
  \and
  Simone Brugiapaglia\thanks{Department of Mathematics \& Statistics, Concordia University, Montr\'eal, QC, Canada; \texttt{simone.brugiapaglia@concordia.ca}.}
}

\maketitle

\begin{abstract}
We show that, in a precise sense, a broad class of feedforward neural networks learn (have finite sample complexity) in the PAC model: every fixed finite feedforward architecture whose layers are definable in an o-minimal structure has finite sample complexity in the agnostic PAC setting, even with unbounded parameters. This covers standard fixed-size MLPs, CNNs, GNNs, and transformers with fixed sequence length, together with the operations and layers typically used in such architectures, including linear projections, residual connections, attention mechanisms, pooling layers, normalization layers, and admissible positional encodings. Hence, distribution-free learnability for modern non-recurrent architectures is not an exceptional property of particular activations or architecture-specific VC arguments, but a consequence of tame feedforward computation. Our results reposition finite-sample PAC learnability as a baseline rather than a differentiator: they shift the focus of architectural comparison toward inductive biases, symmetries and geometric priors, scalability, and optimization behaviour.
\end{abstract}


\section{Introduction}
\label{s:Intro}

Modern neural architectures appear strikingly diverse, ranging from classical multilayer perceptrons~(MLPs) to graph neural networks~(GNNs) and contemporary transformers.  Each major architectural innovation introduces new layers, symmetries, and ``basic computational units,'' and has, in turn, sparked a corresponding line of architecture-specific statistical theory proving that the resulting model class can learn. Examples include guarantees for piecewise-linear MLPs~\citep{bartlett2019nearly}, message-passing GNNs~\citep{GargJegelkaJaakkola2020,MorrisGeertsTonshoffGrohe2023}, and transformer architectures~\citep{EdelmanGoelKakadeZhang2022,TraugerTewari2024,limmer2024higher}.

From the viewpoint of computation, however, these architectures share a common form.  At inference time, they are finite feedforward programs: directed acyclic computational graphs whose nodes execute elementary parametrized operations; cf.\ circuit-complexity theory~\citep{JunkaBooleanBook,jukna2023tropical,MR3308677}.  This observation suggests a simple question.

\begin{questionbox}
\centering
\emph{Do all standard fixed finite feedforward neural networks have finite sample complexity?}
\end{questionbox}

We answer this question \emph{affirmatively} for the broad class of feedforward architectures whose primitive operations are definable in an o-minimal structure. In simple terms: \textit{modern artificial neural networks are mathematically well-behaved enough to be theoretically learnable.} More precisely, we prove that every fixed finite feedforward architecture assembled from definable layers/gates and arranged in a directed acyclic computational graph has finite agnostic PAC sample complexity for both classification (Theorem~\ref{thrm:class}) and regression (Theorem~\ref{thrm:reg}). This holds even when their parameter space is unbounded.  

Throughout the paper, a ``fixed network'' or ``fixed architecture'' means that the computational graph, input and output dimensions, widths, depths, sequence lengths, and the number and placement of trainable parameters are fixed. The numerical values of the parameters themselves are not fixed; they are allowed to vary over their prescribed definable parameter spaces (free to vary during training).

We verify that the definability assumption is mild by checking it for the standard fixed finite building blocks of MLPs, CNNs, GNNs, and transformers, including common activations, convolution, attention, normalization, pooling, and admissible positional encodings.
Here, learnability is statistical and distribution-free: finite uniform-convergence, equivalently finite agnostic PAC sample complexity via ERM.  We do not claim efficient ERM or any optimization guarantee.

\paragraph{Takeaways.} Our results complement the ongoing work in the foundations of statistical learning theory seeking to characterize the exact conditions under which PAC learning is possible by uncovering sharp learning-theoretic dimensions; see, e.g.,~\cite{attias2023optimal,bressan2025fine}.  
Here, we demonstrate that most architectures satisfy the well-known necessary conditions of finite fat-shattering and finite VC dimension, whose sufficiency dates back at least to~\cite{AlonBenDavidCesaBianchiHaussler_PDGeneralization}.
Thus, our \emph{theoretical takeaway} is a simple, virtually always guaranteed \textbf{finite-sample-complexity} result under a unified framework, relying on clarified connections between fat-shattering dimension~\cite{BartlettKulkarniPosner_CoveringFatShattering_1997}, VC dimension, and o-minimal structures in \emph{mathematical logic}~\cite{PillaySteinhorn1986DefinableSetsOrderedStructuresI}, particularly model theory~\cite{vandenDries1998tame,coste1999introduction,speissegger1999pfaffian}.  

The \emph{practical takeaway} is that finite sample complexity is too coarse to distinguish standard feedforward architectures.  Once a model is fixed, finite, feedforward, and built from definable primitives, distribution-free learnability follows automatically.  Thus, model selection should be guided by finer criteria: problem-specific inductive bias, symmetry and geometric priors, as studied in geometric deep learning~\cite{Bronstein2016GeometricDL,bronstein2021geometric,borde2025mathematicalfoundationsgeometricdeep}, as well as scalability~\cite{kaplan2020scalinglawsneurallanguage} and optimization behavior.

\section{Background and Preliminaries}
\label{s:Background}

One of the most important bridges between mathematical logic and learning theory is shattering.  In PAC learning, finite VC dimension is the fundamental combinatorial certificate for distribution-free learnability~\citep{BlumerEhrenfeuchtHausslerWarmuth1989,AnthonyBartlett1999}; in model theory, the same phenomenon appears as the independence property for uniformly definable families~\citep{Laskowski1992,ChaseFreitag2019}.  Early neural-network theory already exploited this connection: tame real geometry was used to prove finiteness results for sigmoidal networks~\citep{MacintyreSontag1993}, polynomial VC bounds for Pfaffian networks~\citep{KarpinskiMacintyre1997}, and general bounds for real-parametrized concept classes generated by elementary operations~\citep{GoldbergJerrum1995}; see also~\citep{BartlettMaass2003}.  Our work revisits this logic-learning bridge in the language of modern architectures: a fixed feedforward network is a finite composition of parametrized gates, so joint definability of the gates propagates to the whole model, while o-minimal tameness rules out unbounded shattering.  Thus, model theory supplies a clean statistical certificate: definable feedforward computation has finite sample complexity.

\subsection{Notation and Terminology}

We begin by fixing the notation and terminology.  We denote the indicator function of a set $A$ by $I_A$.  We also use $I$ to convert Boolean variables into $\{0,1\}$-valued variables, so that
$
    I(\mathrm{True})=1
    \,\text{and}\,
    I(\mathrm{False})=0.
$
An interval means a set of one of the forms
$
    (a,b), (-\infty,a), \,\text{or}\, (b,\infty),
$
where $a,b\in\reals$.  For every $n\in\mathbb{N}_+$, we equip $\reals^n$ with its usual Euclidean topology.

\begin{definition}[Directed Acyclic Graph (DAG)]
\label{def:dag}
A \emph{directed graph} is a pair $G=(V,E)$, where $V$ is a non-empty set of \emph{nodes} and
$
    E\subseteq V^2\setminus \mSet{(v,v)}_{v\in V}
$
is a set of directed edges.  A \emph{directed path} from $v$ to $w$ is a finite sequence of directed edges
$
    (v,v_1),(v_1,v_2),\dots,(v_k,w)\in E.
$
The graph $G$ is called \emph{acyclic} if there is no non-empty directed path from any node to itself.
\end{definition}

\subsection{Background in Mathematical Logic: O-Minimality}
\label{s:ModGeo__ss:oMinimality}

We now review the background required to formulate our results.  Standard learning-theoretic notions used later, such as VC dimension and pseudo-dimension, are recalled in Appendix~\ref{s:MoreBackground__ss:LearningTheory}.
We begin with the basic language of o-minimal geometry; see, for instance,~\citep{vandenDries1998tame,coste1999introduction}.  The central object is a structure on the real field whose definable subsets form a tame class of sets: stable under the usual logical and geometric operations.

\begin{definition}[O-Minimal Structure Expanding the Real Field]
\label{def:o-minimal-structure}
A \emph{structure expanding the real field} is a sequence
$
    \mathcal{S}=(S_n)_{n\in\mathbb{N}_+},
$
where each $S_n$ is a collection of subsets of $\reals^n$, satisfying the following axioms:
\begin{enumerate}
    \item[(i)] Every semi-algebraic subset of $\reals^n$ belongs to $S_n$.
    
    \item[(ii)] For every $n\in\mathbb{N}_+$, the collection $S_n$ is a Boolean subalgebra of $\mathcal{P}(\reals^n)$.
    
    \item[(iii)] If $A\in S_m$ and $B\in S_n$, then $A\times B\in S_{m+n}$.
    
    \item[(iv)] If $p:\reals^{n+1}\to\reals^n$ denotes the projection onto the first $n$ coordinates and $A\in S_{n+1}$, then $p(A)\in S_n$.
\end{enumerate}
The elements of $S_n$ are called the \emph{$\mathcal{S}$-definable subsets} of $\reals^n$. The structure $\mathcal{S}$ is called \emph{o-minimal} if \emph{every set in $S_1$ is a finite union of points and intervals}.
\end{definition}

\begin{example}[Semi-Algebraic + (Global) Exponential]
The real exponential field $\mathbb{R}_{\exp}\eqdef(\mathbb{R},<,+,\cdot,0,1,\exp)$ is o-minimal by~\cite{Wilkie1996}.      
\end{example}
More broadly, we may add in all restricted analytic functions.
\begin{example}[Restricted Analytic Functions + (Global) Exponentials]
The structure $\mathbb{R}_{\mathrm{an},\exp}$, obtained by adding all restricted analytic functions and the global exponential map, is also o-minimal and contains the standard analytic operations used in softmax-type layers by~\cite{vandenDriesMiller1996}.    
\end{example}
Furthermore, we may add in Pfaffian functions and still preserve o-minimality.  Briefly, Pfaffian functions form a robust class of tame functions: their graphs may bend, grow, and interact non-linearly, but cannot oscillate with unrestricted complexity (cf.\  Appendix~\ref{s:MoreBackground__ss:MoreLogic} for a precise definition of Pfaffianity). 
The implication of learnability for neural networks with Pfaffian activation functions was considered in~\cite{KarpinskiMacintyre1997} and more recently for GNNs in~\cite{d2024vc}.
\begin{example}[Pfaffian Closures]
The Pfaffian closure of an o-minimal structure is obtained by adjoining the tame solution sets of definable Pfaffian differential systems.  More precisely, the Pfaffian closure $\mathcal{P}(\mathcal{S})$ of any o-minimal expansion $\mathcal{S}$ of the real field is again o-minimal, and contains the Pfaffian functions generated from $\mathcal{S}$-definable $C^1$ data~\cite{speissegger1999pfaffian}.
\end{example}
A function is well-behaved, in the o-minimal sense, if its graph is definable in the ambient structure.
\begin{definition}[Definable Map]
\label{def:definable-map}
Let $A \subset \mathbb{R}^n$.  
A map
$
f : A \to \mathbb{R}^p
$
is called \emph{definable} if its graph
$
\Gamma_f \;\eqdef\; \{(x,f(x)) \in \mathbb{R}^{n+p} : x \in A\}
$
is a definable subset of $\mathbb{R}^{n+p}$.
\end{definition}
In our examples section, Section~\ref{s:Examples}, we show that most neural networks are definable in the o-minimal structures above.  
We do so by verifying that the standard components used in deep learning are definable therein, and then using the fact that compositions of definable maps are again definable.  
Definable functions enjoy several closure properties: their linear combinations, coordinate-wise concatenations, and compositions are all again definable in $\mathcal{S}$; cf.~\citep[Chapter 1.3]{coste1999introduction}.

\subsection{Definable (Feedforward) Neural Networks}
\label{s:Background__ss:Neural_Networks}

As illustrated in Figure~\ref{fig:ANN}, feedforward neural-network computations are naturally organized by a DAG, in close analogy with the circuit-complexity literature, including Boolean circuits~\cite{JunkaBooleanBook,chiang2025transformers,li2026certifiable}, tropical circuits~\cite{jukna2023tropical}, probabilistic circuits~\cite{choi2020probabilistic,zhang2021probabilistic}, and algebraic circuits~\cite{wang2024compositional}; see also~\cite{kratsios2025quantifying}.  Here, we work with a broader class of feedforward computational graphs, subsuming most modern deep learning architectures as well as these semi-classical models of computational complexity theory.

We fix a family of finite-dimensional parametric maps, which will serve as the source of non-linearities in our deep learning models.  We refer to this family as our \textit{gate dictionary}
\begin{equation}
\label{eq:layer_dict}
    \mathbb{G}
\subseteq
    \bigcup_{n,p,m=1}^{\infty}
    \big\{
        g:\mathbb{R}^n\times\mathbb{R}^p\to\mathbb{R}^m
    \big\}
.
\end{equation}
For any
$
    \mathbb{G}\ni g:\mathbb{R}^n\times\mathbb{R}^p\to\mathbb{R}^m,
$
we interpret the first $n$ coordinates as \emph{inputs} and the last $p$ coordinates as \emph{parameters}, and we write
$
    g_{\theta}\eqdef g(\cdot,\theta)
$
for every fixed parameter $\theta\in\mathbb{R}^p$.

\begin{figure}[htp!]
    \centering
    \begin{subfigure}{0.45\linewidth}
        \centering
        \includegraphics[width=\linewidth]{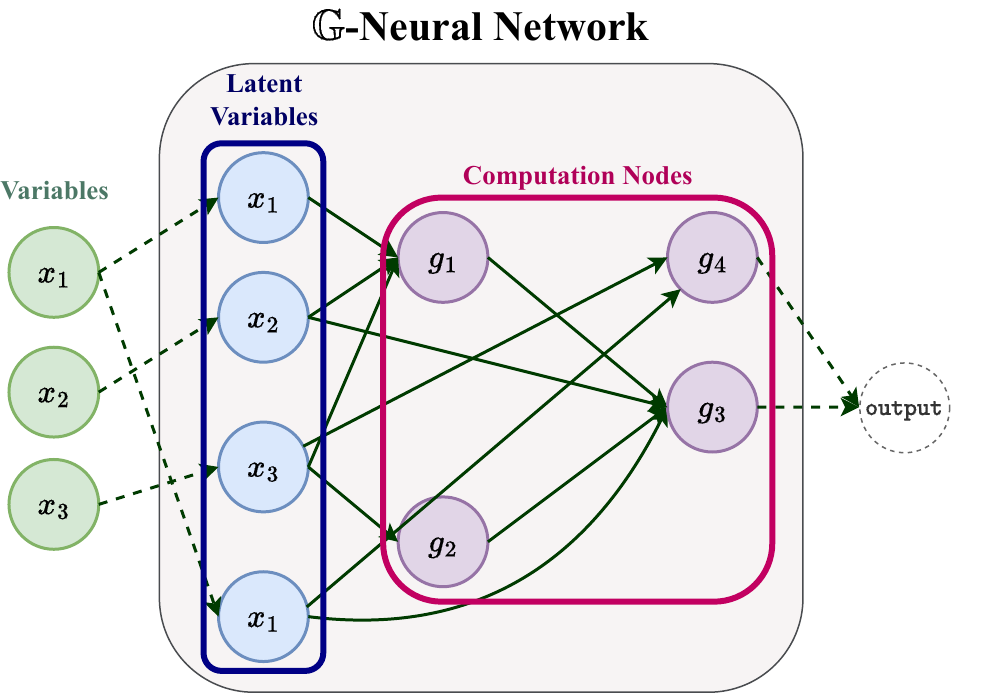}
        \caption{General feedforward $\mathbb{G}$-neural network}
        \label{fig:ANN_general}
    \end{subfigure}
    ~
    \begin{subfigure}{0.45\linewidth}
        \centering
        \includegraphics[width=\linewidth]{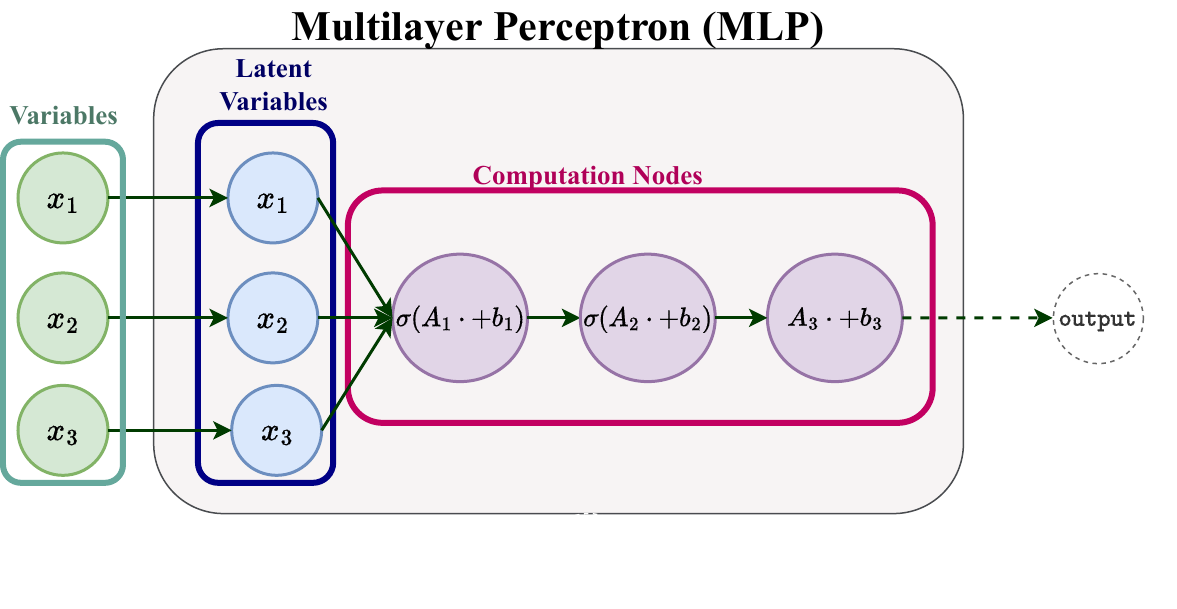}
        \caption{Special case: MLP}
        \label{fig:ANN_MLP}
    \end{subfigure}
    \caption{\textbf{$\mathbb{G}$-NN; cf.\ Definition~\ref{defn:ANN} (left)}: The structure of a \emph{general} non-recursive (feedforward) neural network ($\mathbb{G}$-NN).
    Parameterizable/trainable computations (gates) chosen from the \textit{gate dictionary} $\mathbb{G}$~\eqref{eq:layer_dict} are executed in the order specified by a DAG. \textbf{The ``standardized'' \textit{MLP} special case (right):}
    The standard multilayer perceptron (MLP) is recovered as a special case. The MLP model arises when the computational DAG is a directed line graph and all gates are either of the form $\sigma(A\cdot +b)$ or with $A\cdot + b$ in the final layer.
    }
    \label{fig:ANN}
\end{figure}

Let $G = (V,E)$ be a DAG. For any $v\in V$, the \textit{parents} of $v$ are the nodes
$
    \operatorname{Pa}_G(v)
    \eqdef
        \mSet{u\in V:(u,v)\in E}
.
$
Nodes without parents are called \emph{input} nodes and are denoted by $\operatorname{in}(G)$; all non-input nodes are called \emph{computation} nodes and are denoted by $\operatorname{comp}(G)$.  Nodes without children are called \emph{output} nodes and are denoted by $\operatorname{out}(G)$.
Every DAG induces a natural partial order (\textit{i.e., the direction of computation during the forward pass}) on its nodes: we write $u\lesssim v$ if either $u=v$ or there is a directed path from $u$ to $v$; see, e.g.,~\citep[Proposition~2.1.3]{bang2008digraphs}.  

\begin{definition}[$\mathbb{G}$-Neural Network]
\label{defn:ANN}
Fix dimensions $d_{in},P,d_{out}\in \mathbb{N}_+$ and a gate dictionary $\mathbb{G}$, cf.~\eqref{eq:layer_dict}. 
A map
$
    f:\mathbb{R}^{d_{in}}\times \mathbb{R}^P\to \mathbb{R}^{d_{out}}
$
is called a $\mathbb{G}$-\emph{neural network} if there exist a finite DAG
$
    G=(V,E),
$
a binary lifting channel $\Pi\in \{0,1\}^{|\operatorname{in}(G)|\times d_{in}}$, gates
$
    \mathbb{G}\ni g_v:\mathbb{R}^{M_v}\times\mathbb{R}^{p_v}\to\mathbb{R}^{m_v},
$
$
v\in\operatorname{comp}(G),
$
and a linear readout matrix
$
    A_{\operatorname{out}}\in \mathbb{R}^{d_{out}\times M_{\operatorname{out}}},
$
such that $f$ can be iteratively represented as follows.
First, set
$
    z^0\eqdef \Pi x
$
for each $x\in\mathbb{R}^{d_{in}}$, and identify the coordinates of $z^0$ with the input nodes of $G$.  Then, for every $v\in\operatorname{comp}(G)$, define the \emph{latent variables}
\begin{equation}
\label{eq:composition}
    z_v
    \eqdef
        (g_v)_{\theta_v}
        \big(
            (z_u)_{u\in\operatorname{Pa}_G(v)}
        \big),
\end{equation}
where the parent states are concatenated in any fixed order compatible with the partial order on $G$, and
$
    M_v
    \eqdef
        \sum_{u\in\operatorname{Pa}_G(v)} m_u
$
is the total dimension of the parent state.  Finally, the \emph{output}
\begin{equation}
\label{eq:readout}
    f_{\theta}(x)
    \eqdef
        A_{\operatorname{out}}
        (z_v)_{v\in\operatorname{out}(G)},
\end{equation}
where
$
    M_{\operatorname{out}}
    \eqdef
        \sum_{v\in\operatorname{out}(G)} m_v .
$
We write
$
        \theta
    \eqdef
        \bigl(A_{\operatorname{out}},(\theta_v)_{v\in\operatorname{comp}(G)}\bigr)
        \in\mathbb{R}^P,
$
for its parameter vector, where
$
    P
    =
        d_{out}M_{\operatorname{out}}
        +
        \sum_{v\in\operatorname{comp}(G)} p_v .
$
The \emph{domain} of $f$, denoted $\operatorname{dom}(f)$, consists of all $x\in \mathbb{R}^{d_{in}}$ for which the compositions in~\eqref{eq:composition} and~\eqref{eq:readout} are well-defined.
\end{definition}

\paragraph{The Key Insight: Joint Definability in Inputs and Parameters.}
\label{key_insight}
The core of this work is to show that, if a model family is \textit{jointly definable} in its inputs and parameters, and if the loss function is definable in the same structure, then the resulting model necessarily has finite sample complexity.
\begin{proposition}[Definable \emph{Gates} Imply Definable $\mathbb{G}$-\emph{Neural Networks}]
\label{prop:key_insight}
Fix dimensions $d_{in},P,d_{out}\in \mathbb{N}_+$ and an o-minimal structure $\mathcal{S}$ expanding the real field.  
If every gate in $\mathbb{G}$ is definable in $\mathcal{S}$, then every $\mathbb{G}$-NN is definable in $\mathcal{S}$, jointly in its inputs and parameters.
\end{proposition}

Proposition~\ref{prop:key_insight} motivates our only key assumption, Assumption~\ref{ass:bigass}, whose mildness is truly communicated via our broad list of practical examples in Section~\ref{s:Examples}.
\begin{assumption}[Definable Gates in an O-Minimal Structure]
\label{ass:bigass}
Fix an o-minimal structure $\mathcal{S}$ expanding the real field and suppose that every gate $g\in \mathbb{G}$, cf.~\eqref{eq:layer_dict}, is definable in $\mathcal{S}$.
\end{assumption}

\section{Main Statistical Guarantees}
\label{s:Stats}

We begin with the classification formulation of our general statistical guarantee.
\begin{theorem}[Definability Implies Learning Is Possible---Binary Classification Case]
\label{thrm:class}
Suppose Assumption~\ref{ass:bigass} holds.
Let $n,N\in \mathbb{N}_+$, $f:\mathbb{R}^n\times \mathbb{R}^P\to \mathbb{R}$ be a $\mathbb{G}$-NN, let $\mathbb{P}$ be a Borel probability measure on $\mathbb{R}^n\times \{0,1\}$, and let $((X_i,Y_i))_{i=1}^N$ be i.i.d.\ random variables with law $\mathbb{P}$.
\hfill\\
\noindent
There exist constants%
\footnote{\label{note1}Where $C$ is universal and $K$ depends only on the fixed architecture, the defining formulas of its gates, and, in the regression case, the loss $\ell$, but not on $N$, $\varepsilon$, $\delta$, $\mathbb{P}$, or the magnitude of the parameters.}~%
$C,K>0$ such that: for every error $\varepsilon>0$ and every failure probability $0<\delta \le 1$ 
\[
    \sup_{\theta\in \mathbb{R}^P}\,
    \biggl|
        \mathbb{E}_{(X,Y)\sim \mathbb{P}}\big[
            I(I_{(0,\infty)}\circ f_{\theta}(X)=Y)
        \big]
        -
        \frac1{N}\,\sum_{i=1}^N\,
            I(I_{(0,\infty)}\circ f_{\theta}(X_i)=Y_i)
    \biggr|
    \le
    \varepsilon
\]
holds with probability at least $1-\delta$, provided that 
\[
N
\ge
   C\,\tfrac{K+\log(1/\delta)}{\varepsilon^2}
.
\]
\end{theorem}
We emphasize that Theorem~\ref{thrm:class} imposes \emph{no} weight constraints. In particular, it applies to the class of all $\operatorname{ReLU}$-MLPs with fixed depth and width but \emph{unbounded} parameters: for such classes, standard chaining-based covering arguments conducted over the parameter space cannot yield a finite upper bound on the sample size $N$ (equivalently, a finite complexity bound), since the relevant parameter-space covering numbers are infinite; see, e.g., \citep[Theorem 14.15]{petersen2024mathematical}. Consequently, existing approaches that permit unbounded parameters typically proceed model-by-model and can be technically demanding, cf.~\cite{bartlett2019nearly} for $\operatorname{ReLU}$-MLPs of fixed depth and width. 
By contrast, Theorem~\ref{thrm:class} gives this conclusion automatically for every fixed finite feedforward architecture built from the standard definable primitives verified in Section~\ref{s:Examples}.

The next result is a regression analogue of the preceding one. It holds for bounded definable losses, including \textit{clipped} or normalized versions of MSE and MAE, and KL-type losses on definable domains where the divergence is finite and bounded.
The takeaway is the same: virtually every reasonable deep learning model class will learn to generalize on \textit{regression} problems. The practical use case is likewise the same: we obtain a single, unified guarantee that removes the need for case-by-case analyses to reach this learnability conclusion. Moreover, we do so without imposing parameter constraints that are divorced from real-world AI practice.
\begin{theorem}[Definability Implies Learning Is Possible---Regression Case]
\label{thrm:reg}  
Suppose Assumption~\ref{ass:bigass} holds.
Let $m,n,N\in \mathbb{N}_+$, $f:\mathbb{R}^n\times \mathbb{R}^P\to \mathbb{R}^m$ be a $\mathbb{G}$-NN, let $\mathbb{P}$ be a Borel probability measure on $\mathbb{R}^n\times \mathbb{R}^m$, let $((X_i,Y_i))_{i=1}^N$ be i.i.d.\ random variables with law $\mathbb{P}$,
and let $\ell:\mathbb{R}^m\times \mathbb{R}^m\to [0,1]$ be a definable loss function (in the same o-minimal structure $\mathcal{S}$).
\hfill\\
\noindent
There exist constants$^1$~%
$C,K>0$ (not depending on $N$) such that: for every error $\varepsilon>0$ and every failure probability $0<\delta \le 1$ 
\[
    \sup_{\theta\in \mathbb{R}^P}\,
    \biggl|
        \mathbb{E}_{(X,Y)\sim \mathbb{P}}\big[
            \ell(f_{\theta}(X),Y)
        \big]
        -
        \frac1{N}\,\sum_{i=1}^N\,
            \ell(f_{\theta}(X_i),Y_i)
    \biggr|
    \le
    \varepsilon
\]
holds with probability at least $1-\delta$, provided that 
\[
N
\ge
   C\,\tfrac{K\ln^2(K/\varepsilon)+\ln(1/\delta)}{\varepsilon^2}
.
\]
\end{theorem}

\paragraph{Intuition of Proofs: What Definability Does.}
The main results are derived in three steps.  First, Proposition~\ref{prop:key_insight} shows that any $\mathbb{G}$-neural network is jointly definable in its parameters and inputs; whence the composite map 
$
    ((x,y),\theta)\mapsto I(I_{(0,\infty)}\circ f_\theta(x) = y)
$ 
(resp.\ $((x,y),\theta)\mapsto \ell(f_\theta(x) , y)$ in the regression case) 
is jointly definable in $(x,y,\theta)$.  Next, we use this to show that this composite map has a bounded number of osculations at any scale, implying that $I_{(0,\infty)}\circ f_{\theta}$ (resp.\ $f_{\theta}$) cannot (resp.\ fat-)shatter arbitrarily many points at any scale; whence the family must have a finite VC (resp.\ fat-shattering) dimension as the parameter $\theta$ varies.
Finally, we use classical results of~\cite{Valiant1984,BlumerEhrenfeuchtHausslerWarmuth1989} (resp.~\cite{AlonBenDavidCesaBianchiHaussler_PDGeneralization}) to deduce finite sample-complexity bounds in the classification (resp.\ regression) case.

\paragraph{Scope and Limitations.}
Our result is qualitative and statistical.  It applies to each fixed finite feedforward architecture, with fixed input dimension, fixed depth, fixed width, fixed graph size, and fixed sequence length.  It does not provide rate-sharp constants (which are necessarily architecture-specific), efficient ERM algorithms, optimization guarantees, or bounds uniform over growing model size.  
This excludes \textit{genuinely} \emph{recurrent architectures}, such as RNNs~\cite{Elman1990FindingStructureInTime,HochreiterSchmidhuber1997LongShortTermMemory}, and recurrent/looped transformers~\cite{Dehghani2019Universaltransformers,Giannou2023Loopedtransformers}, which are covered in our theory only by unrolling their iterates up to a fixed \textbf{time horizon}.
Thus, the result should be read as a universal finite-sample-complexity certificate for fixed tame feedforward computation, not as a complete theory of trainability or scaling.

\section{From Theory to Practice: A Large Class of Neural Networks Are Definable}
\label{s:Examples}
We demonstrate the power of our result by establishing the \textbf{generality and scope} of Assumption~\ref{ass:bigass}. It suffices to show that most components used to build a modern deep learning model satisfies this assumption; Proposition~\ref{prop:key_insight} directly extends the conclusion to the models themselves.


\subsection{Neural-Network Building Blocks}

\paragraph{Affine and Polynomial Maps}
The affine functions core to classical MLPs (linear layers) and CNNs (convolutional kernels) are \textit{quadratic polynomial} functions jointly of their inputs and parameters, i.e., $(A,b,x)\mapsto Ax+b$,
are among the simplest examples of definable maps; more generally, the same is
true of polynomial maps, which cover polynomial feature maps (e.g.,
$x\mapsto (x_i x_j)_{i,j}$ or all monomials up to a fixed degree) and polynomial
activations (e.g., $t\mapsto t^2$ or $t\mapsto t^3$). Since the graph of a polynomial map is defined by polynomial equations, polynomial maps are semi-algebraic and hence definable in the real field.

\begin{proposition}[Definability of Polynomial Maps]
\label{prop:polynomial_maps_definable}
Let $n,m\in\mathbb{N}_+$, and let
$
    P:\mathbb{R}^n\to\mathbb{R}^m
$
be a polynomial map.  Then $P$ is semi-algebraic.  In particular, $P$ is definable in the real field $(\mathbb{R},+,\cdot,<)$, and hence in every o-minimal expansion of the real field.
\end{proposition}

\begin{corollary}[Definability of Linear Layers in MLPs]
\label{cor:linear_affine_layers_definable}
Fix $n,m\in\mathbb{N}_+$.  The affine layer $L:\mathbb{R}^n\times \mathbb{R}^{m\times n}\times \mathbb{R}^m
    \to \mathbb{R}^m, L(x,A,b)\eqdef Ax+b,$ is a polynomial map jointly in $(x,A,b)$.  In particular, it is semi-algebraic and definable in the real field $(\mathbb{R},+,\cdot,<)$, and hence in every o-minimal expansion of the real field.
\end{corollary}

\begin{corollary}[Definability of Convolutional Layers in CNNs]
\label{cor:convolutional_layers_definable}
Fix $n,m,p\in\mathbb{N}_+$.  Let
$
    T:\mathbb{R}^p\to\mathbb{R}^{m\times n}
$
be the linear map that sends a vector of convolutional kernel parameters to the corresponding Toeplitz-type (or doubly block Toeplitz) matrix.  The layer $\operatorname{Conv}: \mathbb{R}^n\times\mathbb{R}^p\times\mathbb{R}^m
    \to
    \mathbb{R}^m,
    \operatorname{Conv}(x,\theta,b)\eqdef T(\theta)x+b,$
is a polynomial map jointly in $(x,\theta,b)$.  It is semi-algebraic and definable in the real field $(\mathbb{R},+,\cdot,<)$, and in every o-minimal expansion of the real field.
\end{corollary}

Residual connections (skip connections), such as those used in Highway
Networks~\cite{srivastava2015highwaynetworks,10.5555/2969442.2969505}, ResNets~\cite{resnet} and transformer blocks~\citep{vaswani2017attention}, fall under this framework too.

\begin{corollary}[Residual and Gated Residual Polynomial Blocks]
\label{cor:residual_gated_polynomial_blocks}
Let $F,G:\mathbb{R}^n\to\mathbb{R}^n$ be polynomial maps. Then the residual
block $x \mapsto x+F(x)$
and the gated residual block $x \mapsto x+G(x)\odot F(x),$ where $\odot$ denotes coordinatewise multiplication (or Hadamard product), are polynomial maps. In
particular, they are semi-algebraic and definable in the real field.
\end{corollary}

\paragraph{Activation Functions in Deep Learning.}
Activation functions provide the basic nonlinearities used to construct
feedforward neural architectures. Rectified Linear Units (ReLUs) became a standard choice in modern deep learning~\citep{nair2010rectified,glorot2011deep}, while other alternatives such as Leaky ReLU~\citep{Maas2013Rectifier}, ELU~\citep{Clevert2016ELU}, Swish (also called SiLU)~\citep{ramachandran2017searching} and GELU~\citep{hendrycks2016gelu}, to mention a few, are also popular. For instance, GELU is used in Gemma 3~\citep{gemmateam2025gemma3technicalreport}, and other gated variants such as SwiGLU have also been widely adopted as part of the transformer blocks~\citep{shazeer2020glu} comprising modern large language models (LLMs) such as DeepSeek-V3~\citep{deepseekai2025deepseekv3technicalreport}, DeepSeek-R1~\citep{Guo_2025}, and Llama-3~\citep{grattafiori2024llama3herdmodels}. Also, learnable spline activations
form the basis of Kolmogorov--Arnold Networks (KANs)~\citep{liu2025kan}.
We show that these common non-linearities are definable: piecewise-polynomial
activations are semi-algebraic, while others like the sigmoid, the hyperbolic tangent, Swish, GELU, SwiGLU and more, are definable
in $\mathbb{R}_{\operatorname{an},\exp}$.

\begin{proposition}[Definability of Modern Activation Functions%
\footnote{A comprehensive version of this proposition is available in the appendix; cf.\ Proposition~\ref{prop:modern_activations_definable}.}]
\label{prop:modern_activations_definable__compressed}
Many activation functions: sigmoid $\sigma(x)\eqdef \tfrac{1}{1+e^{-x}}$, $\tanh(x)$, $\operatorname{softplus}(x)$, $\operatorname{ReLU}(x)$, $\operatorname{LeakyReLU}_{\alpha}(x)\eqdef \max\{x,\alpha\,x\}$, $\operatorname{ELU}_{\alpha}(x)\eqdef xI_{x\ge 0}+\alpha (e^x-1)I_{x<0}$, 
$\operatorname{GELU}_{tanh}(x)$ (the approximation used in PyTorch),
$\operatorname{Swish}_{\beta}(x)\eqdef \tfrac{x}{1+e^{-\beta\,x}}$, $\operatorname{SwiGLU}_{\beta}(x_1,x_2)\eqdef x_1\operatorname{Swish}_{\beta}(x_2)$ are definable in $\mathbb{R}_{\operatorname{an}}$, $\mathbb{R}_{\operatorname{exp}}$, or $\mathbb{R}_{\operatorname{an},\exp}$.
\end{proposition}

\paragraph{Multivariate Gating: Maxout, Winner-Take-All, and Mixtures of Experts.}

Many architectures also use comparison-based multivariate non-linearities, such as maxout, winner-take-all, and hard-routed experts.  These maps are piecewise affine over finitely many regions.

\begin{proposition}[Definability of Multivariate Piecewise-Affine Gating Layers]
\label{prop:piecewise_affine_gating}
Fix $d,m,K\in\mathbb{N}_+$.  The following maps are semi-algebraic, hence definable in $(\mathbb{R},+,\cdot,<)$ and in every o-minimal expansion of the real field.

\begin{enumerate}
    \item[(i)] \textbf{Maxout:}  The map
    $
        (x,(a_k,b_k)_{k=1}^K)
        \mapsto
        \max_{1\le k\le K}(a_k^{\top} x+b_k),
    $
    from $\mathbb{R}^d\times(\mathbb{R}^d\times\mathbb{R})^K$ to $\mathbb{R}$.

    \item[(ii)] \textbf{Winner-take-all}: The  map obtained by partitioning $\mathbb{R}^d$ into finitely many polyhedral cells and applying one affine map $x\mapsto A_kx+b_k$ on each cell.
\end{enumerate}
\end{proposition}



These cover hard-routed mixture-of-experts (MoE) layers, where a finite polyhedral rule selects an affine expert on each cell.  More general MoE layers, including for instance those used in GLaM~\citep{du2022glamefficientscalinglanguage}, Mixtral~\citep{jiang2024mixtral}, and DeepSeek-V3/R1~\citep{deepseekai2025deepseekv3technicalreport,Guo_2025}, are definable whenever their experts and routing weights are definable: top-$k$ selection is semi-algebraic, while softmax routing is definable in $\mathbb{R}_{\exp}$.

\paragraph{Attention Layers.}

Attention mechanisms are a central component of modern deep learning, originating in neural machine translation~\citep{bahdanau2015neural} and becoming the defining operation of transformer architectures through multi-head self-attention~\citep{vaswani2017attention}.

\begin{proposition}[Definability of Multi-Head (self-)Attention]
\label{prop:multihead_attention_definable}
Fix $N,H,d_{in},d_k,d_v,d_{out}\in\mathbb{N}_+$ and $\lambda>0$.  For each
$h\in\{1,\dots,H\}$, let
$W_Q^h,W_K^h\in\mathbb{R}^{d_k\times d_{in}}$ and
$W_V^h\in\mathbb{R}^{d_v\times d_{in}}$, and let
$W_O\in\mathbb{R}^{d_{out}\times Hd_v}$.
For $X\in\mathbb{R}^{N\times d_{in}}$, define
\[
    \operatorname{Attn}_h(X)_n
    \eqdef
        \sum_{m=1}^N
        \alpha^{(h)}_{n,m}(X)\,W_V^hX_m,
    \qquad
    \alpha^{(h)}_{n,m}(X)
    \eqdef
        \frac{
            \exp\!\big(
                \lambda\langle W_Q^hX_n,W_K^hX_m\rangle/\sqrt{d_k}
            \big)
        }{
            \sum_{\ell=1}^N
            \exp\!\big(
                \lambda\langle W_Q^hX_n,W_K^hX_\ell\rangle/\sqrt{d_k}
            \big)
        } .
\]
Then the multi-head attention map
$(X,(W_Q^h,W_K^h,W_V^h)_{h=1}^H,W_O)\mapsto \operatorname{MHA}(X)
\eqdef
W_O
\big(
\operatorname{Attn}_1(X)_n,\dots,
\operatorname{Attn}_H(X)_n
\big)$,
is definable in $\mathbb{R}_{\exp}$, jointly in $X$ and the parameters.
\end{proposition}

Extending the result above to cross-attention is trivial, see Corollary~\ref{cor:multihead_cross_attention_definable}. The same is true for sliding-window attention (Corollary~\ref{cor:sliding_window_attention_definable}) which is typical in modern long-context architectures. For example, the Gemma 4 model card describes a hybrid attention mechanism interleaving local sliding-window attention with full global attention~\citep{google2026gemma4modelcard}. A similar softmax-weighted aggregation principle (linear additive self-attention~\citep{saezdeocarizborde2024elucidating}) also underlies graph attention networks (GATs), where attention is restricted to fixed node neighbourhoods~\citep{velickovic2018graph}. This then directly implies the GNN version of that result.


\paragraph{Normalization Layers.}

Normalization layers are now standard components of modern deep architectures, beginning with batch normalization~\citep{ioffe2015batch}, and later variants such as layer normalization~\citep{ba2016layer}, instance normalization~\citep{ulyanov2016instance}, group normalization~\citep{wu2018group}, and RMS normalization~\citep{zhang2019root}.  From the o-minimal viewpoint, these layers introduce no obstruction: after fixing the finite input dimension, parameters, and stabilization constant $\varepsilon>0$, the usual normalization maps are semi-algebraic.

\begin{proposition}[Definability of Standard Normalization Layers]
\label{prop:normalization_layers_definable}
Fix $d\in\mathbb{N}_+$, $\varepsilon>0$, parameters $\gamma,\beta\in\mathbb{R}^d$, and a partition $\mathcal{G}$ of $\{1,\dots,d\}$.  For each $i\in\{1,\dots,d\}$, let $G(i)\in\mathcal{G}$ denote the unique block containing $i$, and define
$
    \mu_G(x)
    \eqdef
        \frac{1}{|G|}
        \sum_{j\in G}x_j
$ and $
    \sigma_G^2(x)
    \eqdef
        \frac{1}{|G|}
        \sum_{j\in G}\big(x_j-\mu_G(x)\big)^2 
.
$
Then, the normalization map $\operatorname{Norm}_{\gamma,\beta,\varepsilon}:\mathbb{R}^d\to\mathbb{R}^d$ defined coordinate-wise by
\[
\operatorname{Norm}_{\gamma,\beta,\varepsilon}(x)_i
    \eqdef
        \gamma_i
        \tfrac{x_i-\mu_{G(i)}(x)}     {\sqrt{\sigma_{G(i)}^2(x)+\varepsilon}}
        +\beta_i
\]
is definable in every o-minimal expansion of the real field.
\end{proposition}

We now consider sophisticated definable layers, possible by definable choice; cf.~\cite[Theorem 3.1]{coste1999introduction}.

\paragraph{Embeddings and Fourier Positional-Encodings.}

Embedding and positional-encoding layers are standard components of transformer-type architectures. Embedding layers with a fixed finite vocabulary are definable because they are
simply lookup maps from a finite set into Euclidean space.

\begin{proposition}[Definability of Embedding Layers]
\label{prop:embedding_layers_definable}
Let $N,d\in\mathbb{N}_+$.  The parametrized embedding map
\[
    \operatorname{Emb}:
    \{1,\dots,N\}\times\mathbb{R}^{N\times d}
    \to
    \mathbb{R}^d,
    \qquad
    \operatorname{Emb}(i,E)
    \eqdef
    E_{i,:},
\]
where $E_{i,:}$ denotes the $i$th row of $E$, is semi-algebraic.  In particular,
it is definable in the real field $(\mathbb{R},+,\cdot,<)$, and hence in every
o-minimal expansion of the real field.
\end{proposition}

Fourier positional encodings also appear broadly in coordinate-based learning and representation learning~\citep{tancik2020fourier}.  From the o-minimal viewpoint, the only subtlety is that the global sine and cosine functions are not definable in $\mathbb{R}_{\mathrm{an},\exp}$.  Thus, Fourier positional encodings must be interpreted either on a fixed finite set of positions, or on a fixed bounded position domain.  Under this standard bounded-context interpretation, they are definable in $\mathbb{R}_{\mathrm{an}}$.

\begin{definition}[\textit{Bounded} Fourier Positional Encoding]
\label{def:bounded_fourier_positional_encoding}
Let $D\subseteq\mathbb{R}^r$ be a bounded definable set, let $M\in\mathbb{N}_+$, and fix frequencies $\omega_1,\dots,\omega_M\in\mathbb{R}^r$ and phases $\varphi_1,\dots,\varphi_M\in\mathbb{R}$.  The associated bounded Fourier positional encoding is the map
$
    \operatorname{PE}_{\Omega,\varphi}:D\to\mathbb{R}^{2M}
$
defined by
$$
    \operatorname{PE}_{\Omega,\varphi}(t)
    \eqdef
    \Big(
        \sin\big(2\pi\langle \omega_j,t\rangle+\varphi_j\big),
        \cos\big(2\pi\langle \omega_j,t\rangle+\varphi_j\big)
    \Big)_{j=1}^M .
$$
\end{definition}
For standard sinusoidal positional encodings (and in RoPE~\cite{su2024roformer}), the frequencies are fixed deterministically.  If frequencies or phases are trainable, we require their parameter domain to be bounded and definable; otherwise, $(t,\omega)\mapsto \sin(\langle \omega,t\rangle)$ may reintroduce global oscillation and fail to be definable.

\begin{proposition}[Definability of Bounded Fourier Positional Encodings]
\label{prop:bounded_fourier_positional_encoding_definable}
Let $D\subseteq\mathbb{R}^r$ be bounded and definable.  For fixed frequencies $\omega_1,\dots,\omega_M\in\mathbb{R}^r$ and fixed phases $\varphi_1,\dots,\varphi_M\in\mathbb{R}$, the bounded Fourier positional encoding
$
    \operatorname{PE}_{\Omega,\varphi}:D\to\mathbb{R}^{2M}
$
is definable in $\mathbb{R}_{\mathrm{an}}$, and hence in $\mathbb{R}_{\mathrm{an},\exp}$.

If $D=\{1,\dots,N\}$ is a fixed finite set of positions, then $\operatorname{PE}_{\Omega,\varphi}$ is semi-algebraic.
\end{proposition}

\paragraph{Pooling Layers.}

Pooling layers are among the classical mechanisms used to build local invariance in hierarchical vision architectures, going back at least to the local aggregation mechanisms of the Neocognitron~\citep{fukushima1980neocognitron}, the trainable sub-sampling layers of LeNet-style convolutional neural networks~\citep{lecun1998gradient}, and the MAX-like operations used in hierarchical object-recognition models~\citep{riesenhuber1999hierarchical}. Their widespread use in large-scale convolutional neural networks was further reinforced by architectures such as AlexNet~\citep{krizhevsky2012imagenet}. Additionally, they are also common in GNNs~\citep{bronstein2021geometric}. From the o-minimal viewpoint, these layers are harmless: average pooling is affine, while max-pooling is semi-algebraic.

\begin{proposition}[Definability of Pooling Layers]
\label{prop:pooling_layers_definable}
Fix $d\in\mathbb{N}_+$.  The following are definable: (i) Average Pooling: The average-pooling map $P_{\operatorname{avg}}:\mathbb{R}^d\to\mathbb{R}$ defined by
    $
        P_{\operatorname{avg}}(x_1,\dots,x_d)
        \eqdef
            \frac{1}{d}\sum_{i=1}^d x_i
    $ is definable in any o-minimal structure. (ii) Max Pooling: The max-pooling map $P_{\max}:\mathbb{R}^d\to\mathbb{R}$ defined by
    $
        P_{\max}(x_1,\dots,x_d)
        \eqdef
            \max_{1\le i\le d}x_i
    $ is definable in any o-minimal structure.
\end{proposition}

Additionally, beyond operations, we also complement our results with deep equilibrium layers.

\paragraph{Deep Equilibrium Layers.}

Deep equilibrium models~\citep{bai2019deep} belong to the broader class of implicit-depth models, where the output of a layer is not obtained by applying a prescribed finite composition of maps, but rather by solving an equation whose solution defines the hidden representation.  This viewpoint also underlies several influential architectures and implicit layers: including neural ordinary differential equations~(ODEs) \citep{chen2018neural} and differentiable optimization layers \citep{amos2017optnet}.  

\begin{proposition}[Definability of Deep Equilibrium Layers]
\label{prop:deq_layers_definable}
Let $\mathfrak{S}$ be an o-minimal expansion of the real field, let
$X\subseteq\mathbb{R}^{d_x}$ and $Z\subseteq\mathbb{R}^{d_z}$ be
definable sets, and let
$
    F:X\times Z\to Z
$ and $
    G:X\times Z\to\mathbb{R}^{d_y}
$
be definable maps.  Assume that, for every $x\in X$, there
is a $z_x\in Z$ satisfying
$
    z_x = F(x,z_x).
$
Then the deep equilibrium layer
$
    \operatorname{DEQ}(x)
    \eqdef
        G(x,z_x)
$
is definable\footnote{In particular, if $F$ and $G$ are built from affine maps, coordinate-wise
polynomial or piecewise-polynomial activations, normalization layers, and
softmax attention blocks, then $\operatorname{DEQ}$ is definable in
$\mathbb{R}_{\exp}$.}.
\end{proposition}

\subsection{How to Apply our Result: Feedforward Architectures}
\label{s:Examples__ss:Architectures}
We now package the preceding layer-wise definability results into architecture-level corollaries.  The aim is not to strengthen Proposition~\ref{prop:key_insight} or Theorems~\ref{thrm:class} and~\ref{thrm:reg}, but to spell out their consequences for standard fixed finite architectures.  

\begin{corollary}[MLPs Are Definable and Have Finite Sample Complexity]
\label{cor:fixed_mlp_sample_complexity}
Fix $L\in\mathbb{N}_+$ and widths
$
    d_0,d_1,\dots,d_L,d_{L+1}\in\mathbb{N}_+ .
$
For each $l\in\{0,\cdots,L-1\}$, let $\sigma_l$ be as in Proposition~\ref{prop:modern_activations_definable__compressed}, and let $P\eqdef \sum_{l=0}^{L} d_{l+1}(d_l+1)$.  
Then, the MLP $f_{\theta}:\mathbb{R}^{d_0}\times \mathbb{R}^P\to \mathbb{R}^{d_{L+1}}$ given, for any $(x,\theta)$, by
\[
    f_{\theta}(x)
    \eqdef 
    A_L x^{(L)}+b_L,
    \quad
    x^{(l+1)}\eqdef \sigma_l\bullet
    \big(
        A_l x^{(l)}+b_l
    \big)
    \mbox{ for }l\in \{0,\cdots,L-1\}
    \mbox{ and }
    \quad
    x^{(0)}\eqdef x
\]
where $\bullet$ denotes component-wise composition
 and 
 $\theta \eqdef (\operatorname{vec}(A_l),b_l)_{l=0}^{L}$
, is definable in $\mathbb{R}_{\operatorname{an},\exp}$.
\hfill\\
In particular, the conclusions of Theorems~\ref{thrm:class} and~\ref{thrm:reg} apply.
\end{corollary}

\begin{corollary}[Transformers are Definable and Have Finite Sample Complexity]
\label{cor:fixed_transformer_sample_complexity}
Fix $N,L,T,d_{\operatorname{out}}\in\mathbb{N}_+$, fix $d_0\in\mathbb{N}_+$, and, for each $l\in[L]_+$, let
$
    H_l,d_l,d_{k,l},d_{v,l},d_{\operatorname{ff},l}\in\mathbb{N}_+ .
$
Let $\lambda,\varepsilon>0$.  Let $E\in\mathbb{R}^{N\times d_0}$ be an embedding parameter, and let
$\operatorname{Emb}$ be as in Proposition~\ref{prop:embedding_layers_definable}, let
$
    \operatorname{PE}_{\Omega,\varphi}:\{1,\dots,T\}\to\mathbb{R}^{d_0}
$
be as in Proposition~\ref{prop:bounded_fourier_positional_encoding_definable}, for fixed $\Omega$ and $\varphi$, and, for each $l\in[L]_+$, let $\operatorname{Norm}_{l}$ denote $\operatorname{Norm}_{\gamma_l,\beta_l,\varepsilon}$ as in Proposition~\ref{prop:normalization_layers_definable} (two per block), for fixed $\gamma_l,\beta_l\in\mathbb{R}^{d_l}$.  For each $l\in[L]_+$, let $\sigma_l$ be as in Proposition~\ref{prop:modern_activations_definable__compressed}.  Set
\[
    P\eqdef 
    Nd_0
    +
    \sum_{l=1}^{L}
    \Big(
        H_l d_{l-1}(2d_{k,l}+d_{v,l})
        + H_l d_{v,l}d_l
        + d_{l-1}d_l
        + d_{\operatorname{ff},l}(2d_l+1)
        + d_l
    \Big)
    + d_{\operatorname{out}}(d_L+1).
\]
Then, the transformer
$
    f_{\theta}:\{1,\dots,N\}^T\times\mathbb{R}^P
    \to \mathbb{R}^{T\times d_{\operatorname{out}}}
$
given, for any $(X,\theta)$, by
\[
\begin{aligned}
    f_{\theta}(X)
&\eqdef
    Z^{(L)}W_{\operatorname{out}}
    +\mathbf{1}_T b_{\operatorname{out}}^{\top}
\\
    Z^{(l)}
&\eqdef
    \operatorname{Norm}_{l,2}^{\bullet}
    \Big(
        Y^{(l)}
        +
        \big(
            \sigma_l\bullet
            \big(
                Y^{(l)} W^1_l+\mathbf{1}_T(b^1_l)^{\top}
            \big)W^2_l
            +
            \mathbf{1}_T(b^2_l)^{\top}
        \big)
    \Big)
\,\,
&&\mbox{ for } l\in[L]_+
\\
    Y^{(l)}
&\eqdef
    \operatorname{Norm}_{l,1}^{\bullet}
    \big(
        Z^{(l-1)}W^R_l
        +
        \operatorname{MHA}_l(Z^{(l-1)})
    \big)
\,\,
&&\mbox{ for } l\in[L]_+
\\
    Z^{(0)}_t
&\eqdef
    \operatorname{Emb}(X_t,E)
    +
    \operatorname{PE}_{\Omega,\varphi}(t)
\,\,
&&\mbox{ for } t\in[T]_+,
\end{aligned}
\]
where $\operatorname{Norm}_{l}^{\bullet}$ denotes row-wise normalization; where
$
    \operatorname{MHA}_l(Z)
    \eqdef
    \Big(
        \bigoplus_{h=1}^{H_l}
        \operatorname{Attn}_{l,h}(Z)
    \Big)W^O_l
$,
\[
    \operatorname{Attn}_{l,h}(Z)_n
    \eqdef
    \sum_{m=1}^{T}
    \alpha^{(l,h)}_{n,m}(Z)\,Z_mW^V_{l,h},
    \quad
    \alpha^{(l,h)}_{n,m}(Z)
    \eqdef
    \frac{
        \exp\big(
            \lambda\langle Z_nW^Q_{l,h},Z_mW^K_{l,h}\rangle/\sqrt{d_{k,l}}
        \big)
    }{
        \sum_{r=1}^{T}
        \exp\big(
            \lambda\langle Z_nW^Q_{l,h},Z_rW^K_{l,h}\rangle/\sqrt{d_{k,l}}
        \big)
    },
\]
where  $\oplus$ denotes concatenation over attention heads, $\bullet$ denotes component-wise composition, and
$
    \theta
\eqdef
    \big(E,
        (W^Q_{l,h},W^K_{l,h},W^V_{l,h})_{l,h},
        (W^O_l,W^R_l,W^1_l,b^1_l,W^2_l,b^2_l)_{l=1}^{L},
        W_{\operatorname{out}},b_{\operatorname{out}}
    \big)
    \in\mathbb{R}^P.
$
\hfill\\
Then, $f_{\theta}$ is definable in $\mathbb{R}_{\operatorname{an},\exp}$; and the conclusions of Theorems~\ref{thrm:class} and~\ref{thrm:reg} apply.
\end{corollary}

Analogous conclusions can also be derived for pre-LN transformers~\citep{xiong2020layernormalizationtransformerarchitecture} and MoEs, CNNs, most GNNs, and most other modern feedforward architectures in a similar fashion, using the previous propositions and main theorems.

\vspace{-10pt}

\section{Conclusion}
\label{s:Conclusion}

We showed that distribution-free generalization is a universal consequence of tame feedforward computation.  More precisely, every fixed finite feedforward architecture (numerical values of the parameters themselves are not fixed: they are trainable) whose layers are definable in an o-minimal structure has finite sample complexity in the agnostic PAC setting for classification (Theorem~\ref{thrm:class}) and for bounded definable regression losses (Theorem~\ref{thrm:reg}), even when its parameters are unbounded.  This covers the standard non-recurrent architectures used in modern deep learning, including MLPs, CNNs, finite-size GNNs, and transformers, to name a few. The main conceptual consequence is that PAC learnability is \textit{not} a \textit{fragile property} requiring architecture-specific VC arguments, nor is it tied to particular activations or parameter bounds.  Instead, it follows automatically from the definable structure shared by essentially all finite feedforward models used in practice.  This shifts the role of generalization theory: for modern non-recurrent architectures, finite sample complexity should be viewed as a baseline property, not as the feature that distinguishes one architecture from another. 

\paragraph{The Practical Takeaway.}
Accordingly, the meaningful differences between architectures must be sought elsewhere: in their inductive biases, symmetries and geometric prior, optimization geometry, scalability, and computational efficiency. Therefore o-minimality provides a unifying language for separating what is generic about deep learning from what is genuinely architecture-specific.

\subsection{Future Work}
\label{appendix:Future Work}
We conclude by outlining promising directions for future research.

\paragraph{Quantitative Versions via $^{\sharp}o$-Minimality.}
Our theory shows how o-minimality implies finite sample complexity in the PAC-learning model.  We expect that precise quantitative versions of our qualitative results can be derived by relying on the emerging notion of $^{\sharp}$o-minimality~\cite{BinyaminiNovikov2023Tameness,BinyaminiNovikovZak2024Wilkie,BinyaminiNovikovZak2026Sharply}, since this new model-theoretic tool allows one to control the number of cells/osculations of a jointly definable function through notions of degree and format, akin to the corresponding complexity theory for Pfaffian functions; cf.~\cite{Khovanskii_Fewnomials}.

\paragraph{Definability Beyond PAC Learning.}
It would be interesting to apply our proof strategy to obtain general learning guarantees in other models, where bounds on the number of definable cells, or osculations, yield uniform control on the relevant learning-theoretic dimensions; e.g., the Littlestone dimension for online learning~\cite{Littlestone1988LearningQuickly}, the $\gamma$-Natarajan dimension~\cite{simon1997bounds} for multi-class classification, or other extensions for robust PAC~\cite{gourdeau2021hardness} and computable PAC learning; cf.~\cite{agarwal2021open,delle2023find}.





\bibliographystyle{plainnat}
\bibliography{Bookkeeping/3_References}

\appendix

\clearpage

\section{Additional Background}
\label{s:MoreBackground}

In this appendix we provide additional mathematical background in Pfaffian functions, learning theory, and VC dimensions.

\subsection{Pfaffian Functions: A Smooth Tame Class}
\label{s:MoreBackground__ss:MoreLogic}
Pfaffian functions, introduced by Khovanskii in his theory of fewnomials~\citep{Khovanskii_Fewnomials}, form a robust class of real-analytic functions whose derivatives satisfy triangular polynomial differential equations.  They provide a smooth analogue of semi-algebraic functions: their format controls their algebraic-differential complexity, while the main result of~\citep{speissegger1999pfaffian} ensures that adjoining Pfaffian solutions to an o-minimal structure preserves o-minimality.  Thus Pfaffian functions are both analytically regular and model-theoretically tame.

\begin{Def}[Pfaffian Chain]
\label{def:pfaffian_chain}
Let $U\subseteq\reals^n$ be a nonempty open domain.  A Pfaffian chain of length $q$ and degree $D$ on $U$ is a sequence of real-analytic functions
$
    f_1,\dots,f_q:U\to\reals
$
such that, for every $1\le i\le q$ and every $1\le j\le n$, there exists a polynomial
$
    P_{i,j}\in\reals[X_1,\dots,X_n,Y_1,\dots,Y_i]
$
of degree at most $D$ satisfying
$$
    \frac{\partial f_i(\bar{x})}{\partial x_j}
    =
    P_{i,j}\big(\bar{x},f_1(\bar{x}),\dots,f_i(\bar{x})\big),
    \qquad
    \bar{x}\in U .
$$
\end{Def}

\begin{Def}[Pfaffian Function of Format $(q,D,d)$]
\label{defn:PfaffianFunctionFormat}
Let $U\subseteq\reals^n$ be a nonempty open domain.  A function $g:U\to\reals$ is Pfaffian of format $(q,D,d)$ on $U$ if there exist a Pfaffian chain
$
    f_1,\dots,f_q:U\to\reals
$
of length (or order) $q$ and degree $D$, and a polynomial
$
    Q\in\reals[X_1,\dots,X_n,Y_1,\dots,Y_q]
$
of degree at most $d$, such that
$$
    g(\bar{x})
    =
    Q\big(\bar{x},f_1(\bar{x}),\dots,f_q(\bar{x})\big),
    \qquad
    \bar{x}\in U .
$$
By convention, if $F:U\to\reals^m$, then $F$ is Pfaffian of format at most $(q,D,d)$ on $U$ if each coordinate function $F_1,\dots,F_m$ is Pfaffian of format at most $(q,D,d)$ on $U$.
\end{Def}

In particular, Pfaffian functions are definable in an o-minimal expansion of the real field.  Since real-analyticity is part of the definition of a Pfaffian chain, every Pfaffian function is real-analytic, and hence $C^\infty$, on its domain.

For illustrative purposes, following~\cite[Example 2.3]{GabrielovVorobjov_NoetherianPfaffianComplexity_2004}, we record several elementary examples of Pfaffian functions.

\begin{example}[Elementary functions]
\label{ex:Pfaffians}
The functions $e^{\beta \cdot}$, for any $\beta\in\reals$, $\log(|\cdot|)$, $\arctan(\cdot)$, and every real-analytic branch of an algebraic function are Pfaffian on their domains of definition.  In particular, $x\mapsto x^\alpha$, for rational $\alpha$, is Pfaffian on any interval on which it is real-valued and real-analytic.
\end{example}

\begin{example}[Sinusoidal functions on bounded domains]
\label{ex:bounded_sine_pfaffian}
The function $t\mapsto \sin(t)$ is not Pfaffian on all of $\reals$, but its restriction to any bounded interval is Pfaffian.  Consequently, the three-variable function
$
    (x,w,b)\mapsto \sin(wx+b)
$
is Pfaffian on every bounded open domain $U\subseteq\reals^3$.

Indeed, set
$
    u(x,w,b)\eqdef wx+b.
$
Since $U$ is bounded, $u(U)$ is bounded.  Choose $k\in\mathbb{N}_+$ large enough so that $u(U)/k$ is contained in an interval of length strictly smaller than $2\pi$.  Then there exists $c\in\reals$ such that
$$
    \frac{u(U)}{k}-c
    \subset
        (-\pi,\pi).
$$
Define, for $(x,w,b)\in U$,
\begin{align*}
    f_1(x,w,b)
    &\eqdef
        \tan\Big(
            \frac{u(x,w,b)/k-c}{2}
        \Big),\\
    f_2(x,w,b)
    &\eqdef
        \frac{1}{1+f_1(x,w,b)^2}.
\end{align*}
The functions $f_1$ and $f_2$ are real-analytic on $U$.  Moreover, for each $z\in\{x,w,b\}$,
$$
    \frac{\partial f_1}{\partial z}
    =
        \frac{1}{2k}
        \big(1+f_1^2\big)
        \frac{\partial u}{\partial z}
    \in
        \reals[x,w,b,f_1],
$$
and
$$
    \frac{\partial f_2}{\partial z}
    =
        -2f_1f_2^2
        \frac{\partial f_1}{\partial z}
    =
        -\frac{1}{k}
        f_1f_2^2
        \big(1+f_1^2\big)
        \frac{\partial u}{\partial z}
    \in
        \reals[x,w,b,f_1,f_2].
$$
Thus $(f_1,f_2)$ is a Pfaffian chain on $U$; for instance, it has length $2$ and degree at most $6$.

Now let
$
    \theta(x,w,b)\eqdef u(x,w,b)/k-c.
$
By the half-angle identities,
$$
    \sin(\theta)
    =
        2f_1f_2,
    \qquad
    \cos(\theta)
    =
        (1-f_1^2)f_2.
$$
Therefore, by the angle-sum formula,
$$
    \sin\Big(\frac{u(x,w,b)}{k}\Big)
    =
        2\cos(c)f_1f_2
        +
        \sin(c)(1-f_1^2)f_2,
$$
and
$$
    \cos\Big(\frac{u(x,w,b)}{k}\Big)
    =
        \cos(c)(1-f_1^2)f_2
        -
        2\sin(c)f_1f_2.
$$
Finally, if $U_{k-1}$ denotes the Chebyshev polynomial of the second kind, then
$$
    \sin(u(x,w,b))
    =
        \sin\Big(k\,\frac{u(x,w,b)}{k}\Big)
    =
        \sin\Big(\frac{u(x,w,b)}{k}\Big)
        U_{k-1}\Big(
            \cos\Big(\frac{u(x,w,b)}{k}\Big)
        \Big).
$$
Since both $\sin(u/k)$ and $\cos(u/k)$ are polynomial expressions in $f_1$ and $f_2$ of degree at most $3$, it follows that $\sin(wx+b)$ is Pfaffian on $U$, with format at most $(2,6,3k)$.
\end{example}

As discussed in~\cite[page 111]{Khovanskii_Fewnomials}, Pfaffian functions are closed under sums, products, derivatives, composition, and reciprocals of nonvanishing Pfaffian functions.  They are also stable under solving suitable Pfaffian equations.  Quantitative bounds on the order and degree of Pfaffian functions constructed in this way are available; see, for instance,~\citep[Lemmas 2.4 and 2.5]{GabrielovVorobjov_NoetherianPfaffianComplexity_2004}.

\subsection{Learning Theory}
\label{s:MoreBackground__ss:LearningTheory}

Consider a uniform class of scalar-valued functions 
\[\calF=\{f(\bar{x},\bar{\theta}):\reals^n\rightarrow\reals ^1, \bar{\theta}\in \reals^p\}.\]

We say that a set $A=\{\bar a_1,\ldots, \bar a_d\}$ is \emph{pseudo-shattered} by $\calF$ if and only if there exists a set of real numbers $\{r_1,\ldots, r_d\}$ such that for every subset $A'\subseteq A$, there is $\bar{\theta}_{A'}\in \reals^p$ such that 
\[\bar{a}_i\in A' \Leftrightarrow  f(\bar{a}_i,\bar{\theta}_{A'})>r_i.\]

The \emph{pseudo-dimension} of $\calF$, $\pdim(\calF)$, is the size of the largest set that can be pseudo-shattered by $\calF$. If we allow non-scalar functions,
\[\calF=\{f(\bar{x},\bar{\theta}):\reals^n\rightarrow\reals ^m, \bar{\theta}\in \reals^p\},\]
we will define the pseudo-dimension to be the component-wise minimum pseudo-dimension, and so we will restrict our attention to scalar-valued functions.

The pseudo-dimension is given with a strict inequality, but the bi-implication in this definition still gives the same encoding of the characteristic function of each subset $A'\subseteq A$, perhaps with a little less ``wiggle room''.  This motivates the notion of fat-shattering dimension of~\cite{DBLP:journals/jcss/KearnsS94}, defined as follows.

Fix a tolerance $\gamma>0$.  We say that a set $A=\{\bar a_1,\ldots,\bar a_d\}\subseteq \reals^n$
is \emph{$\gamma$-fat shattered} by $\calF$ if and only if there exists a set of
real numbers $\{r_1,\ldots,r_d\}$ such that for every subset $A'\subseteq A$,
there exists $\bar{\theta}_{A'}\in \reals^p$ such that, for every
$i\in\{1,\ldots,d\}$,
\[
\bar a_i\in A'
\mbox{ implies that: }
f(\bar a_i,\bar\theta_{A'}) \ge r_i+\gamma
\qquad\text{\textit{and}}\qquad
\bar a_i\notin A'
\mbox{ implies that: }
f(\bar a_i,\bar\theta_{A'}) \le r_i-\gamma.
\]
\noindent
The \emph{$\gamma$-fat shattering dimension} of $\calF$ is
\[
\fat_{\gamma}(\calF)
\eqdef
\sup\Big\{\, d\in\mathbb{N}\ :\ \exists A\subseteq \reals^n,\ |A|=d,\
A\ \text{is $\gamma$-fat shattered by }\calF \Big\},
\]
with the convention that $\fat_{\gamma}(\calF)=\infty$ if sets of arbitrarily
large size are $\gamma$-fat shattered by $\calF$.

\subsubsection{Reducing to the VC Dimension}
\label{s:Intro__ss:LearningTheory___sss:VCRedux}
We will work in $\reals$ equipped with the usual language of ordered rings, with a symbol ``$<$'' for strict inequality. 
Consider a uniform family of scalar-valued functions 
\[\calF=\{f(\bar{x},\bar{\theta}):\reals^n\rightarrow\reals ^1, \bar{\theta}\in \reals^p\}.\]
Then $\pdim(\calF)\geq V$ if and only if there is a set $A=\{(\bar{a}_1,r_1),\ldots, (\bar{a}_V,r_V)\}$ such that for any subset $A'\subseteq A$, there is a tuple $\bar{\theta}_{A'}\in \reals^p$ that encodes the characteristic function of $A'$:
\[(\bar{a}_i,r_i)\in A'\Leftrightarrow f(\bar{a}_i,\bar{\theta}_{A'})-r_i>0.\]
In other words, the pseudo-dimension of $\calF$ is at least $V$ if and only if the VC dimension of the class of uniformly definable sets 
\[\calC_\calF=\{f(\bar{x},\bar{\theta})-z>0 : \bar{\theta}\in \reals^p\}\]
is at least $V$. Thus, we may bound the pseudo-dimension of $\calF$ by bounding the size $V$ of a set that is shattered by $\calC_\calF$.

\section{Proofs}
\label{s:Proofs}

In this appendix we provide the proofs for the main theoretical results discussed in this paper, as well as for concrete neural network building blocks and architectures.

\subsection{Proof of the Key Insight}
\label{s:Proof_KeyInsight}
\begin{proof}[{Proof of Proposition~\ref{prop:key_insight}}]
Let $f:\mathbb{R}^{d_{in}}\times\mathbb{R}^P\to\mathbb{R}^{d_{out}}$ be a $\mathbb{G}$-NN.  By Definition~\ref{defn:ANN}, there is a finite DAG $G=(V,E)$, a binary lifting channel $\Pi$, gates
$
    \mathbb{G}\ni g_v:\mathbb{R}^{M_v}\times\mathbb{R}^{p_v}\to\mathbb{R}^{m_v},
    \qquad v\in\operatorname{comp}(G),
$
and a linear readout $A_{\operatorname{out}}$ such that $f$ is obtained by propagating states along $G$.

We prove, by induction along the partial order of $G$, that every node state is definable as a function of $(x,\theta)$.  For the input nodes, this is immediate: the lifted input
$
    z^0=\Pi x
$
is a linear, hence semi-algebraic, function of $x$.  Since $\mathcal{S}$ expands the real field, every semi-algebraic map is definable in $\mathcal{S}$.

Now fix $v\in\operatorname{comp}(G)$ and assume that, for every parent $u\in\operatorname{Pa}_G(v)$, the state $z_u$ is definable as a function of $(x,\theta)$.  Then the concatenated parent state
$
    (z_u)_{u\in\operatorname{Pa}_G(v)}
$
is definable, since finite products and coordinate projections preserve definability.  The parameter projection $\theta\mapsto \theta_v$ is also semi-algebraic, hence definable.  Therefore the map
$
    (x,\theta)
    \longmapsto
        \Big(
            (z_u)_{u\in\operatorname{Pa}_G(v)},
            \theta_v
        \Big)
$
is definable.  Since $g_v$ is definable in $\mathcal{S}$, closure of definable maps under composition implies that
$
    z_v(x,\theta)
    =
        (g_v)_{\theta_v}
        \big(
            (z_u(x,\theta))_{u\in\operatorname{Pa}_G(v)}
        \big)
$
is definable in $\mathcal{S}$.

Since $G$ is finite and acyclic, this induction reaches every computation node.  In particular, every output-node state $z_v$, $v\in\operatorname{out}(G)$, is definable in $(x,\theta)$.  The final readout
$
    f_{\theta}(x)
    =
        A_{\operatorname{out}}
        (z_v(x,\theta))_{v\in\operatorname{out}(G)}
$
is obtained from definable maps using only finite products, coordinate projections, addition, and multiplication.  Hence it is definable in $\mathcal{S}$.  Thus $f$ is jointly definable in its inputs and parameters, as claimed.
\end{proof}
\begin{remark}
Moreover, if some gates are only defined on definable domains, then the same induction shows that $\operatorname{dom}(f)$ is definable, and $f$ is definable on $\operatorname{dom}(f)$.
\end{remark}

\subsection{Proofs of General o-Minimal Bounds}
\label{s:Proofs__ss:ominimalGeneral}

The proof does not proceed through the usual metric-complexity route.  Our main technical tool is a uniform cell decomposition over the full parameterization, exploiting the o-minimal tameness, or fixed-format Pfaffian regularity, of the parametrized learning rule.  Consequently, we do not need to truncate the parameter space to $[-M,M]^p$, estimate covering numbers or Rademacher complexities there, and then control how these estimates propagate through the parameter-to-realization map.  That approach is poorly matched to unbounded parameter spaces, since the metric size of $[-M,M]^p$ diverges as $M\to\infty$.  The definable route captures a different phenomenon: uniformly over all parameters in $\mathbb{R}^p$, the model family cannot oscillate arbitrarily often.  Equivalently, its realizable sign patterns admit a finite uniform cell decomposition.  This is the geometric mechanism behind finite sample complexity with unbounded parameters.

\begin{lemma}[Jointly-Definable Families of Classifiers Have Tame Pre-Images]
\label{lem:finite_connected_components}
Let $m,n\in \mathbb{N}_+$, let $X\subseteq \mathbb{R}^m$ be non-empty and definable, and let
$f:\mathbb{R}^n\times X\to \mathbb{R}$ be definable (in the same o-minimal setting).
Then there exists a constant $K\in\mathbb{N}_+$ such that, for every $\theta\in X$,
the set
\[
\big(I_{(0,\infty)}\circ f(\cdot,\theta)\big)^{-1}[\{0\}]
=
\{x\in\mathbb{R}^n:\ f(x,\theta)\le 0\}
\]
has at most $K$ connected components.
\end{lemma}
\begin{proof}[{Proof of Lemma~\ref{lem:finite_connected_components}}]
Define the jointly-definable classifier
\[
g:\mathbb{R}^n\times X\to \{0,1\},
\qquad
g(x,\theta)\eqdef I_{(0,\infty)}(f(x,\theta)).
\]
Since $(0,\infty)$ is definable and definable maps are closed under composition
(cf.~\citep[Exercise~1.11]{coste1999introduction}), the map $g$ is definable; hence so is
\[
A \eqdef g^{-1}[\{0\}] = \{(x,\theta)\in \mathbb{R}^n\times X:\ f(x,\theta)\le 0\}.
\]
Let $\pi:\mathbb{R}^n\times X\to X$ denote the coordinate projection $(x,\theta)\mapsto \theta$.
Apply the cell decomposition theorem to $A$ \emph{compatible with} $\pi$
(cf.~\citep[Theorem~2.10]{coste1999introduction}) to obtain finitely many pairwise disjoint definable cells
$C_1,\dots,C_N\subseteq \mathbb{R}^n\times X$ such that
\[
A=\bigsqcup_{j=1}^N C_j,
\]
and such that for each $j$ and each $\theta\in X$ the fibre
\[
(C_j)_\theta \eqdef \{x\in\mathbb{R}^n:\ (x,\theta)\in C_j\}
\]
is either empty or a definable cell in $\mathbb{R}^n$.
Now, fix $\theta\in X$ and set
\[
A_\theta \eqdef \{x\in\mathbb{R}^n:\ (x,\theta)\in A\}
=
\big(I_{(0,\infty)}\circ f(\cdot,\theta)\big)^{-1}[\{0\}].
\]
Then
$
A_\theta=\bigsqcup_{j=1}^N (C_j)_\theta
$; whence, by \cite[Proposition 2.5]{coste1999introduction}, every (non-empty) cell is definably homeomorphic to some Euclidean space,
hence connected. Therefore each $(C_j)_\theta$ is either empty or connected, and so $A_\theta$ has at most $N$
connected components. Taking $K\eqdef N$ completes the proof.
\end{proof}

Now, using our uniform cell decomposition, we are able to infer a VC dimension bound on the class of classifiers.

\begin{lemma}[Uniform Cell Decomposition Implies VC Dimension Bounds]
\label{lem:VC_Dim_Bounds}
Let $m,n\in \mathbb{N}_+$, let $Z\subseteq \mathbb{R}^n$ and $X\subseteq \mathbb{R}^m$ both be non-empty and definable, and let
$f:Z\times X\to \mathbb{R}$ be definable.
Define the family of classifiers
\[
\mathcal{F}^{\star}
\eqdef
\Big\{
Z\ni x\mapsto I_{(0,\infty)}\big(f(x,\theta)\big)
\ :\ \theta\in X
\Big\}.
\]
Then, $\operatorname{VCdim}(\mathcal{F}^{\star})<\infty$.
More precisely, there exists a constant $C>0$ such that, for every $N\in\mathbb{N}_+$ and every
$N$-point set $S\subseteq Z$,
\begin{equation}
\label{eq:uniform_cell}
        \big|\mathcal{F}^{\star}|_S\big|
    \le 
        C\, N^m
    .
\end{equation}
Consequently,
$
    \operatorname{VCdim}(\mathcal{F}^{\star})
<
    \infty
$.
\end{lemma}

\begin{proof}[Proof of Lemma~\ref{lem:VC_Dim_Bounds}]
Fix $N\in\mathbb{N}_+$ and let $S=\{x_i\}_{i=1}^N\subseteq Z$ be an arbitrary $N$-point set.  For each
$i=1,\dots,N$, define the definable subset of parameter space
\[
    X_i
    \eqdef
    \{\theta\in X:\ f(x_i,\theta)>0\}.
\]
A labeling of $S$ induced by $\mathcal{F}^{\star}$ is determined by the membership pattern of $\theta$ in the
sets $X_1,\dots,X_N$.  Equivalently, the number of distinct labelings induced on $S$ is bounded by the number of
non-empty atoms of the Boolean algebra generated by $X_1,\dots,X_N$.

By uniform cell decomposition for definable families~\cite[Theorem~2.10]{coste1999introduction}, there exists a
constant $C>0$, depending only on the definable family induced by $f$ and on $X$, such that this Boolean algebra has
at most $C N^m$ non-empty atoms, uniformly over $N$ and over the choice of $S\subseteq Z$.  Therefore
\[
    \big|\mathcal{F}^{\star}|_S\big|
    \le C N^m.
\]
If $S$ is shattered by $\mathcal{F}^{\star}$, then $\big|\mathcal{F}^{\star}|_S\big|=2^N$.  Hence
\[
    2^N\le C N^m.
\]
Since the exponential function eventually dominates every polynomial, this inequality can hold for only finitely
many $N\in\mathbb{N}_+$.  Thus
$
    \operatorname{VCdim}(\mathcal{F}^{\star})
    \le
    \max\{N\in\mathbb{N}_+:\ 2^N\le C N^m\}
    <\infty
$.
\end{proof}

\begin{remark}[The One-Dimensional Refinement]
\label{rem:VC_Dim_Bounds_One_Dimensional}
In the special case where $Z\subseteq \mathbb{R}$, the preceding conclusion admits a sharper elementary bound.
Indeed, suppose that there exists $K\in\mathbb{N}_+$ such that, for every $\theta\in X$, the set
$
    \{x\in Z:\ f(x,\theta)\le 0\}
$
has at most $K$ connected components.  Then
\[
    \operatorname{VCdim}(\mathcal{F}^{\star})\le 2K.
\]
To see this, define the associated set system
\[
\tilde{\mathcal{F}}
\eqdef
\Big\{
\{x\in Z:\ f(x,\theta)\le 0\}
\ :\ \theta\in X
\Big\}.
\]
Since $x\mapsto I_{(0,\infty)}\big(f(x,\theta)\big)=1$ if and only if
$x\notin \{y\in Z:\ f(y,\theta)\le 0\}$, the classes $\mathcal{F}^{\star}$ and $\tilde{\mathcal{F}}$ induce the same
labelings on every finite sample up to flipping $0$ and $1$.  Hence
\[
    \operatorname{VCdim}(\mathcal{F}^{\star})
    =
    \operatorname{VCdim}(\tilde{\mathcal{F}}).
\]
Now let $S=\{x_1<\cdots<x_{2K+1}\}\subseteq Z$.  Since every set
$\{x\in Z:\ f(x,\theta)\le 0\}$ has at most $K$ connected components, and since connected components in
$\mathbb{R}$ are intervals, its trace on $S$ is a union of at most $K$ blocks of consecutive points.  Therefore it
cannot realize the alternating subset
$
    \{x_1,x_3,x_5,\dots,x_{2K+1}\},
$
which consists of $K+1$ singleton blocks.  Thus no $(2K+1)$-point subset of $Z$ is shattered, and consequently
$
    \operatorname{VCdim}(\mathcal{F}^{\star})\le 2K
$.
\end{remark}

This now allows us to obtain classifier bounds via the core result in VC theory.
We propose and prove the following more general formulation of Theorem~\ref{thrm:class}.
\begin{theorem}
\label{thrm:class___generalized}
Let $m,n\in \mathbb{N}_+$, let $X\subseteq \mathbb{R}^m$ be non-empty and definable, and let
$f:\mathbb{R}^n\times X\to \mathbb{R}$ be definable (in the same o-minimal setting).  Set
\[
\calF^{\star}
 \eqdef  
\big\{ I_{(0,\infty)}\circ f(\cdot,\theta)\ :\ \theta\in X\big\},
\]
so that, for each $\theta\in X$,
\[
\big(I_{(0,\infty)}\circ f(\cdot,\theta)\big)^{-1}[\{0\}]
=
\{x\in\mathbb{R}^n:\ f(x,\theta)\le 0\}
.
\]
Let $K \eqdef \operatorname{VCdim}(\calF^{\star})<\infty$, and let $\ell$ denote the $0$--$1$ loss.
Then, there exists an absolute constant $C>0$ such that: for every Borel probability measure
$\mathbb{P}$ on $\mathbb{R}^n\times \{0,1\}$, every sequence of i.i.d.\ samples
$\{(X_i,Y_i)\}_{i=1}^{\infty}$ with law $\mathbb{P}$, every error $\varepsilon>0$, and every failure
probability $0<\delta \le 1$, the following holds with probability at least $1-\delta$:
\[
    \sup_{h\in \calF^{\star}}\,
    \biggl|
        \mathbb{E}_{(X,Y)\sim \mathbb{P}}\big[
            \ell(h(X),Y)
        \big]
        -
        \frac1{N}\,\sum_{i=1}^N\,
            \ell(h(X_i),Y_i)
    \biggr|
    \le
    \varepsilon
\]
provided that
\[
N
\ge
   C\,\tfrac{K+\log(1/\delta)}{\varepsilon^2}
.
\]
\end{theorem}
\begin{proof}[{Proof of Theorem~\ref{thrm:class___generalized}}]
By Lemma~\ref{lem:VC_Dim_Bounds}, applied with $Z=\mathbb{R}^n$, the class of classifiers
$\calF^{\star}$ has finite VC dimension, which we denote by $K$.
Applying the fundamental theorem of PAC-learning for binary classifiers, as formulated in~\citep[Theorem 6.7 (1)]{shalev2014understanding},
yields the conclusion.
\end{proof}

In a similar vein, pseudo-dimension bounds can be obtained from uniform cell decompositions of the type in Lemma~\ref{lem:VC_Dim_Bounds}.
Indeed, a standard lifting argument applied to the hypothesis class reduces the pseudo-dimension estimate to a VC dimension bound for an associated class of subgraphs. Once this reduction is in place, the proof is immediate: we invoke the preceding lemma that converts uniform cell decompositions into VC bounds, and then apply the main generalization theorem of~\cite{AlonBenDavidCesaBianchiHaussler_PDGeneralization}.  
\begin{lemma}[Standard Lifting Trick]
\label{lem:pdim_fat_via_lift}
Let $m,n\in\mathbb{N}_+$, let $Z\subseteq \mathbb{R}^n$ and $X\subseteq \mathbb{R}^m$ be non-empty, and let
$f:Z\times X\to\mathbb{R}$.
Define the real-valued class
\[
\mathcal{F}\eqdef \Big\{\,x\mapsto f(x,\theta)\ :\ \theta\in X\,\Big\}.
\]
For each $\gamma\ge 0$, define the lifted classifier class on $Z\times\mathbb{R}$ by
\[
\mathcal{H}_\gamma
\eqdef
\Big\{
(x,r)\mapsto I_{(0,\infty)}\big(f(x,\theta)-r-\gamma\big)
\ :\ \theta\in X
\Big\}.
\]
Then, for every $\gamma>0$, we have
$
    \operatorname{fat}_\gamma(\mathcal{F})
    \le
    \operatorname{Pdim}(\mathcal{F})
    \le 
    \operatorname{VCdim}(\mathcal{H}_0)
$.
\end{lemma}
\begin{proof}[Proof of Lemma~\ref{lem:pdim_fat_via_lift}]
For $\gamma\ge 0$ and $\theta\in X$, define
$
h_{\theta}^{(\gamma)}:Z\times\mathbb{R}\to\{0,1\}
$
by
$
h_{\theta}^{(\gamma)}(x,r)\eqdef I_{(0,\infty)}\big(f(x,\theta)-r-\gamma\big),
$
so that $\mathcal{H}_\gamma=\{h_{\theta}^{(\gamma)}:\theta\in X\}$.
\hfill\\
\noindent
Fix $d\in\mathbb{N}_+$.  By definition, $\{x_1,\dots,x_d\}\subseteq Z$ is pseudo-shattered by $\mathcal{F}$ if and only if there exist
$r_1,\dots,r_d\in\mathbb{R}$ such that, for every $b\in\{0,1\}^d$, one can choose $\theta\in X$ with
\[
    I_{(0,\infty)}\big(f(x_i,\theta)-r_i\big)=b_i
\]
for each $i=1,\dots,d$.
Setting $z_i\eqdef (x_i,r_i)\in Z\times\mathbb{R}$, this implies that, for every $b\in\{0,1\}^d$, there exists $\theta\in X$ such that
$
    h_{\theta}^{(0)}(z_i)=b_i
$
for each $i=1,\dots,d$.  Hence $\{z_1,\dots,z_d\}$ is shattered by $\mathcal{H}_0$.  Taking suprema over $d$ yields
$
    \operatorname{Pdim}(\mathcal{F})\le \operatorname{VCdim}(\mathcal{H}_0)
$.
\hfill\\
\noindent
Now fix $\gamma>0$ and suppose that $\{x_1,\dots,x_d\}\subseteq Z$ is $\gamma$-fat shattered by $\mathcal{F}$, witnessed by
$r_1,\dots,r_d\in\mathbb{R}$.  Thus, for every labeling $b\in\{0,1\}^d$, one can choose $\theta\in X$ so that, for each
$i=1,\dots,d$,
$
b_i=1
$
implies
$
f(x_i,\theta)\ge r_i+\gamma
$
and
$
b_i=0
$
implies
$
f(x_i,\theta)\le r_i-\gamma.
$
Since $\gamma>0$, this implies
\[
    I_{(0,\infty)}\big(f(x_i,\theta)-r_i\big)=b_i
\]
for each $i=1,\dots,d$.  Hence $\{x_1,\dots,x_d\}$ is pseudo-shattered by $\mathcal{F}$ with witnesses
$r_1,\dots,r_d$.  Therefore $d\le \operatorname{Pdim}(\mathcal{F})$, and taking suprema over all such $d$ yields
$
    \operatorname{fat}_\gamma(\mathcal{F})\le \operatorname{Pdim}(\mathcal{F})
$.
\end{proof}
In order to obtain our generalization bound from our uniform cell decomposition, guaranteed for definable families of parametric functions, within an o-minimal structure, we only require the following repackaged result of~\cite{AlonBenDavidCesaBianchiHaussler_PDGeneralization}.
\begin{lemma}[Finite Pseudo-Dimension Implies Statistical Guarantees~\cite{AlonBenDavidCesaBianchiHaussler_PDGeneralization}]
\label{lem:Pd_2_Gen}
Let $p,n,m\in \mathbb{N}_+$, and
$\calF=\{f(\bar{x},\bar{\theta}):\mathbb{R}^n\to \mathbb{R}^m \,:\, \bar{\theta}\in \reals^p\}$
be a set of Borel measurable functions, and let $\ell:\mathbb{R}^m\times \mathbb{R}^m\to [0,1]$
be a Borel loss function such that the composite class
\begin{equation}
\label{eq:compositeclass}
\mathcal{F}_{\ell}
 \eqdef 
\{\ell_f:\, \ell_f  =  \ell\circ (f\times \operatorname{Id}_{\mathbb{R}^m}) 
,\, 
f\in \mathcal{F}\}
\end{equation}
has finite pseudo-dimension $\operatorname{Pdim}(\mathcal{F}_{\ell})<\infty$.
\\
\noindent Fix a Borel probability measure $\mathbb{P}$ on $\mathbb{R}^{n+m}$, and let
$(X_1,Y_1),\dots,(X_N,Y_N)$ be i.i.d.\ random variables with law $\mathbb{P}$.  
There is an absolute constant $c_1>1$ such that:
\\
\noindent for every accuracy parameter $\varepsilon>0$ and failure probability $0<\delta \le 1$,
the following holds with probability at least $1-\delta$:
\[
\sup_{f\in \calF}\,
\biggl|
    \mathbb{E}_{(X,Y)\sim \mathbb{P}}\big[
        \ell(f(X),Y)
    \big]
    -
    \frac1{N}\,\sum_{i=1}^N\,
        \ell(f(X_i),Y_i)
\biggr|
\le
\varepsilon
\]
provided that
$
N
\ge
    c_1\left( \frac{1}{\varepsilon^2} \left(
        \operatorname{Pdim}(\mathcal{F}_{\ell})
        \ln^2\left(\frac{
                \operatorname{Pdim}(\mathcal{F}_{\ell})
            }{\varepsilon}\right) + \ln\left(\frac{1}{\delta}\right) \right) \right)
$.
\end{lemma}

\begin{proof}[{Proof of Lemma~\ref{lem:Pd_2_Gen}}]
Since $\operatorname{Pdim}(\mathcal{F}_{\ell})<\infty$ by assumption, 
by~\cite[Theorem 11.13 (i)]{AnthonyBartlett1999}, for every $\gamma>0$ we have
\begin{equation}
    \label{eq:fat_pseudo}
    \fat_{\gamma}\big(
    \mathcal{F}_{\ell}
    \big)
    \le
    \operatorname{Pdim}\big(\mathcal{F}_{\ell} \big)
    .
\end{equation}
By~\cite[Theorem 3.1]{AlonBenDavidCesaBianchiHaussler_PDGeneralization}, we have that
for every error size $\varepsilon>0$ and each failure probability $0<\delta \le 1$ the following holds
\begin{equation}
    \label{eq:AlonBenDavidBound}
    \mathbb{P}\Biggl(
    \sup_{h\in \mathcal{F}_{\ell}}
    \,
    \Big|
    \mathbb{E}_{(X,Y)\sim \mathbb{P}}\big(h(X,Y)\big)
    -
    \frac1{N}\sum_{i=1}^N\,h(X_i,Y_i)
    \Big|
    \le
    \varepsilon
    \Biggr)
    \ge
    1-\delta
\end{equation}
provided that the sample size $N$ satisfies
\begin{equation}
\label{eq:Fatshatteringbound}
\begin{aligned}
    N
&\ge
    c_1\left( \frac{1}{\varepsilon^2} \left(
        \fat_{\varepsilon/32}(\mathcal{F}_{\ell})
        \ln^2\left(\frac{
                \fat_{\varepsilon/32}(\mathcal{F}_{\ell})
            }{\varepsilon}\right) + \ln\left(\frac{1}{\delta}\right) \right) \right)
\end{aligned}
\end{equation}
for some absolute constant $c_1>1$.
Now, the inequalities in~\eqref{eq:fat_pseudo} and~\eqref{eq:Fatshatteringbound} imply that it is sufficient for $N$ to satisfy
\begin{equation}
\label{eq:PseudoDimensionSize}
\begin{aligned}
N
&\ge
c_1\left( \frac{1}{\varepsilon^2} \left(
\operatorname{Pdim}(\mathcal{F}_{\ell})
\ln^2\left(\frac{
\operatorname{Pdim}(\mathcal{F}_{\ell})
}{\varepsilon}\right) + \ln\left(\frac{1}{\delta}\right) \right) \right)
\end{aligned}
\end{equation}
in order for~\eqref{eq:AlonBenDavidBound} to hold.
\end{proof}


\begin{remark}[Lipschitz Losses May Destroy Definability]
Now, if we only knew that the loss function $\ell$ was Lipschitz, without assuming definability in an o-minimal structure, then there would be no guarantee that the pseudo-dimension of the composite class $\calF_{\ell}$ would remain finite even if that of $\calF$ is finite.  For instance, if $n=m=1$, $\calF=\{ax:\, a\in [0,2\pi]\}$ and $\ell(y_1,y_2)\eqdef \tfrac{\sin(y_1)+1}{2}$, then $\calF$ has finite pseudo-dimension, but the composite class need not have finite pseudo-dimension.   
\end{remark}
\begin{remark}
Note also that we did not pass from the fat-shattering dimension to covering number bounds, e.g., using a result such as~\cite[Theorem 2]{BartlettKulkarniPosner_CoveringFatShattering_1997}, which may unnecessarily loosen our bounds.
\end{remark}
We now prove the regression variant of Theorem~\ref{thrm:reg} for definable regressors.
\begin{theorem}[Definability Implies Learning Is Possible---Regression Case]
\label{thrm:definability__2__generalizability___regression}
Let $m,n,p\in \mathbb{N}_+$, let $X\subseteq \mathbb{R}^p$ be non-empty and definable,
let $f:\mathbb{R}^n\times X\to \mathbb{R}^m$ be definable, and let $\ell:\mathbb{R}^m\times \mathbb{R}^m\to [0,1]$ be definable.
Set
\[
\calF\eqdef\{x\mapsto f(x,\theta):\theta\in X\}.
\]
Then there exist constants $C,K>0$ such that, for any Borel probability measure $\mathbb{P}$ on $\mathbb{R}^{n+m}$, any i.i.d.\ random variables
$(X_1,Y_1),\dots,(X_N,Y_N)$ with law $\mathbb{P}$, any accuracy parameter $\varepsilon>0$, and any failure probability $0<\delta \le 1$, the following bound holds with probability at least $1-\delta$:
\hfill\\
\noindent
\[
\sup_{g\in \calF}\,
\biggl|
    \mathbb{E}_{(X,Y)\sim \mathbb{P}}\big[
        \ell(g(X),Y)
    \big]
    -
    \frac1{N}\,\sum_{i=1}^N\,
        \ell(g(X_i),Y_i)
\biggr|
\le
\varepsilon
\]
provided that
$
N
\ge
    C\left( \frac{1}{\varepsilon^2} \left(
        K
        \ln^2\left(\frac{
                K
            }{\varepsilon}\right) + \ln\left(\frac{1}{\delta}\right) \right) \right)
$.
\end{theorem}
\begin{proof}[{Proof of Theorem~\ref{thrm:definability__2__generalizability___regression}}]
Since $f$ is definable, the identity map $\operatorname{Id}_{\mathbb{R}^m}:\mathbb{R}^m\to \mathbb{R}^m$ is definable, $\ell$ is definable, and the products and compositions of definable functions are definable, cf.~\citep[Section 1.3]{coste1999introduction},
the map
$
f_{\ell}:\mathbb{R}^{n+m}\times X\to [0,1]
$ sending any $(x,y)$ to $f_{\ell}((x,y),\theta)\eqdef \ell(f(x,\theta),y)$ 
is definable.
Thus, the conclusion follows from Lemma~\ref{lem:VC_Dim_Bounds} combined with Lemma~\ref{lem:Pd_2_Gen}.
\end{proof}

\subsection{Proofs regarding Definable Neural Network Building Blocks and Architectures}
\label{s:Proofs_Examples}

\begin{proof}[{Proof of Proposition~\ref{prop:polynomial_maps_definable}}]
Write
$
    P=(P_1,\dots,P_m)
$,
where each $P_j:\mathbb{R}^n\to\mathbb{R}$ is a polynomial.  Then
$$
    \operatorname{graph}(P)
    =
    \Big\{
        (x,y)\in\mathbb{R}^{n+m}:
        y_j-P_j(x)=0,\,
        j=1,\dots,m
    \Big\}.
$$
Hence $\operatorname{graph}(P)$ is algebraic, and therefore semi-algebraic.  Thus $P$ is semi-algebraic, so it is definable in $(\mathbb{R},+,\cdot,<)$, and consequently in every o-minimal expansion of the real field.

Finally, polynomials are Pfaffian functions relative to the empty Pfaffian chain.  Therefore each coordinate $P_j$ is Pfaffian, and hence $P$ is Pfaffian.
\end{proof}

\begin{proof}[Proof of Corollary~\ref{cor:linear_affine_layers_definable}]
For each output coordinate $i\in\{1,\dots,m\}$, we have
\[
    L_i(x,A,b)
    =
    \sum_{j=1}^n A_{ij}x_j+b_i .
\]
This is a polynomial function of the coordinates of $(x,A,b)$. Hence $L$ is a polynomial map. The conclusion follows from Proposition~\ref{prop:polynomial_maps_definable}.
\end{proof}

\begin{proof}[Proof of Corollary~\ref{cor:convolutional_layers_definable}]
For each output coordinate $i\in\{1,\dots,m\}$, we have
\[
    \operatorname{Conv}(x,\theta,b)_i
    =
    \sum_{j=1}^n T(\theta)_{ij}x_j+b_i .
\]
Since $T:\mathbb{R}^p\to\mathbb{R}^{m\times n}$ is linear, each entry $T(\theta)_{ij}$ is a linear function of $\theta$. Hence each term $T(\theta)_{ij}x_j$ is polynomial in the coordinates of $(x,\theta,b)$. Therefore each output coordinate is polynomial in $(x,\theta,b)$, and so $\operatorname{Conv}$ is a polynomial map.

The conclusion follows from Proposition~\ref{prop:polynomial_maps_definable}.
\end{proof}

\begin{proof}[Proof of Corollary~\ref{cor:residual_gated_polynomial_blocks}]
Write
\[
    F=(F_1,\dots,F_n)
    \qquad\text{and}\qquad
    G=(G_1,\dots,G_n),
\]
where each $F_i,G_i:\mathbb{R}^n\to\mathbb{R}$ is a polynomial function.

The residual block
\[
    R(x)\eqdef x+F(x)
\]
has coordinates
\[
    R_i(x)=x_i+F_i(x),\qquad i=1,\dots,n.
\]
Since sums of polynomial functions are polynomial, each coordinate $R_i$ is polynomial. Hence $R$ is a polynomial map.

Similarly, the gated residual block
\[
    H(x)\eqdef x+G(x)\odot F(x)
\]
has coordinates
\[
    H_i(x)=x_i+G_i(x)F_i(x),\qquad i=1,\dots,n.
\]
Since products and sums of polynomial functions are polynomial, each coordinate $H_i$ is polynomial. Hence $H$ is also a polynomial map.

Therefore both residual and gated residual polynomial blocks are polynomial maps. By Proposition~\ref{prop:polynomial_maps_definable}, they are semi-algebraic and definable in the real field.
\end{proof}

The following proposition is an extended version of Proposition~\ref{prop:modern_activations_definable__compressed}.

\begin{proposition}[Definability of Modern Activation Functions]
\label{prop:modern_activations_definable}
The following activation functions are definable.

\begin{enumerate}

    \item[(i)] The logistic sigmoid activation
    \[
        \sigma(x)
        \eqdef
        \frac{1}{1+e^{-x}}
    \]
    is Pfaffian on $\mathbb{R}$ of format at most $(1,2,1)$.  In particular, it is definable in the Pfaffian closure of the real field, and also in $\mathbb{R}_{\exp}$.

    \item[(ii)] The hyperbolic tangent activation
\[
    \tanh(x)
    =
    \frac{e^{2x}-1}{e^{2x}+1}
\]
is Pfaffian on $\mathbb{R}$ of format at most $(1,2,1)$. In particular, it is
definable in the Pfaffian closure of the real field, and also in
$\mathbb{R}_{\exp}$.

    \item[(iii)] The softplus activation
\[
    \operatorname{Softplus}(x)
    \eqdef
    \log(1+e^x)
\]
is Pfaffian on \(\mathbb{R}\) of format at most \((2,2,1)\). In particular,
it is definable in the Pfaffian closure of the real field, and also in
\(\mathbb{R}_{\exp}\).
    
    \item[(iv)] Every piecewise-polynomial map
    $
        f:\mathbb{R}^n\to\mathbb{R}^m
    $
    of the form
    $
        f=\sum_{k=1}^K p_k I_{A_k},
    $
    where $p_1,\dots,p_K:\mathbb{R}^n\to\mathbb{R}^m$ are polynomial maps and $A_1,\dots,A_K\subseteq\mathbb{R}^n$ form a finite semi-algebraic partition of $\mathbb{R}^n$, is semi-algebraic.  In particular, it is definable in the real field $(\mathbb{R},+,\cdot,<)$.  
    
    \hfill\\
    In particular, the following standard activations are semi-algebraic and hence definable in the real field:
    \[
        \operatorname{ReLU}(x)=\max\{x,0\}, 
        \qquad
        \operatorname{LeakyReLU}_{\alpha}(x)=\max\{x,\alpha x\},
    \]
    for fixed $\alpha\in\mathbb{R}$, the parametric ReLU
    $
        \operatorname{PReLU}(x,\alpha)=\max\{x,\alpha x\}
    $
    jointly in $(x,\alpha)$, $\operatorname{ReLU}^p$ for $p\in\mathbb{N}_+$, hard-threshold activations, hard-tanh, hard-sigmoid, and finite-knot spline activations, such as those used in KANs.

    \item[(v)] The exponential linear unit
\[
    \operatorname{ELU}_{\alpha}(x)
    \eqdef
    \begin{cases}
        x, & x\ge 0,\\
        \alpha(e^x-1), & x<0,
    \end{cases}
\]
for fixed \(\alpha\in\mathbb{R}\), is definable in \(\mathbb{R}_{\exp}\).
Moreover, the parametrized map
\[
    (x,\alpha)\longmapsto \operatorname{ELU}_{\alpha}(x)
\]
is definable in \(\mathbb{R}_{\exp}\) jointly in \((x,\alpha)\). Consequently,
for fixed constants \(\lambda,\alpha\in\mathbb{R}\), the SELU activation
\[
    \operatorname{SELU}_{\lambda,\alpha}(x)
    \eqdef
    \lambda\,\operatorname{ELU}_{\alpha}(x)
\]
is also definable in \(\mathbb{R}_{\exp}\).

    \item[(vi)] The Gaussian Error Linear Unit
\[
    \operatorname{GELU}(x)
    \eqdef
        x\Phi(x)
    =
        \frac{x}{2}
        \left(
            1+\operatorname{erf}\left(\frac{x}{\sqrt{2}}\right)
        \right)
\]
is Pfaffian on \(\mathbb{R}\) of format at most \((2,2,2)\). In particular,
\(\operatorname{GELU}\) is definable in the Pfaffian closure of the real field.

The commonly used tanh approximation, implemented for example by PyTorch when
\(\texttt{approximate='tanh'}\), is
\[
    \operatorname{GELU}_{\tanh}(x)
    \eqdef
    \frac{x}{2}
    \left(
        1+
        \tanh\left(
            \sqrt{\frac{2}{\pi}}
            \left(x+0.044715x^3\right)
        \right)
    \right).
\]
This approximation is also Pfaffian on \(\mathbb{R}\), and is definable in
\(\mathbb{R}_{\exp}\), hence also in
\(\mathbb{R}_{\operatorname{an},\exp}\).

    \item[(vii)] The Swish activation
    $
        \operatorname{Swish}(x,\beta)
        \eqdef
            \frac{x}{1+e^{-\beta x}}
    $  on 
        $(x,\beta)\in\mathbb{R}\times(0,\infty)$
    is Pfaffian of format at most $(2,4,2)$.  In particular, it is definable in the Pfaffian closure of the real field, and also in $\mathbb{R}_{\exp}$.

    \item[(viii)] The SwiGLU activation
    $$
        \operatorname{SwiGLU}(x_1,x_2,\beta)
        \eqdef
            x_1\,\operatorname{Swish}(x_2,\beta)
        =
            x_1\frac{x_2}{1+e^{-\beta x_2}},
        \qquad
        (x_1,x_2,\beta)\in\mathbb{R}^2\times\mathbb{R}_{>0},
    $$
    is Pfaffian of format at most $(2,4,3)$.  In particular, it is definable in the Pfaffian closure of the real field, and also in $\mathbb{R}_{\exp}$.

    \item[(ix)] The softsign activation
\[
    \operatorname{Softsign}(x)
    \eqdef
    \frac{x}{1+|x|}
\]
is semi-algebraic. In particular, it is definable in the real field
\((\mathbb{R},+,\cdot,<)\), and hence in
\(\mathbb{R}_{\operatorname{an}}\) and
\(\mathbb{R}_{\operatorname{an},\exp}\).

    \item[(x)] The Mish activation
\[
    \operatorname{Mish}(x)
    \eqdef
    x\,\tanh\big(\operatorname{Softplus}(x)\big),
    \qquad
    \operatorname{Softplus}(x)\eqdef \log(1+e^x),
\]
is definable in \(\mathbb{R}_{\exp}\), and hence in
\(\mathbb{R}_{\operatorname{an},\exp}\).
\end{enumerate}
\end{proposition}

\begin{proof}[{Proof of Proposition~\ref{prop:modern_activations_definable}.(i)}]
Let
\[
    f_1(x)\eqdef \sigma(x)=\frac{1}{1+e^{-x}},
    \qquad x\in\mathbb{R}.
\]
Then \(f_1\) is real-analytic on \(\mathbb{R}\). Moreover,
\[
    \frac{d f_1}{dx}(x)
    =
    \frac{e^{-x}}{(1+e^{-x})^2}
    =
    f_1(x)\bigl(1-f_1(x)\bigr).
\]
Equivalently,
\[
    \frac{d f_1}{dx}
    =
    P(f_1),
    \qquad
    P(Y_1)\eqdef Y_1-Y_1^2.
\]
The polynomial \(P\in\mathbb{R}[Y_1]\subseteq\mathbb{R}[X,Y_1]\) has degree \(2\).
Hence \((f_1)\) is a Pfaffian chain on \(\mathbb{R}\) of length \(1\) and degree \(2\).

Finally,
\[
    \sigma(x)=f_1(x)
\]
is obtained from this chain by the polynomial \(Q(Y_1)=Y_1\), of degree \(1\).
Therefore \(\sigma\) is Pfaffian on \(\mathbb{R}\), of format at most \((1,2,1)\).
In particular, \(\sigma\) is definable in the Pfaffian closure of the real field.
Since it is also obtained from field operations and the global exponential map, it is
definable in \(\mathbb{R}_{\exp}\), and hence in \(\mathbb{R}_{\operatorname{an},\exp}\).
\end{proof}

\begin{proof}[{Proof of Proposition~\ref{prop:modern_activations_definable}.(ii)}]
Let
\[
    f_1(x)\eqdef \tanh(x),
    \qquad x\in\mathbb{R}.
\]
Then \(f_1\) is real-analytic on \(\mathbb{R}\). Moreover,
\[
    \frac{d f_1}{dx}(x)
    =
    1-\tanh^2(x)
    =
    1-f_1(x)^2.
\]
Equivalently,
\[
    \frac{d f_1}{dx}
    =
    P(f_1),
    \qquad
    P(Y_1)\eqdef 1-Y_1^2.
\]
The polynomial \(P\in\mathbb{R}[Y_1]\subseteq\mathbb{R}[X,Y_1]\) has degree
\(2\). Hence \((f_1)\) is a Pfaffian chain on \(\mathbb{R}\) of length \(1\)
and degree \(2\).

Finally,
\[
    \tanh(x)=f_1(x)
\]
is obtained from this chain by the polynomial \(Q(Y_1)=Y_1\), of degree \(1\).
Therefore \(\tanh\) is Pfaffian on \(\mathbb{R}\), of format at most
\((1,2,1)\). In particular, \(\tanh\) is definable in the Pfaffian closure of
the real field. Since
\[
    \tanh(x)=\frac{e^{2x}-1}{e^{2x}+1},
\]
and the denominator \(e^{2x}+1\) is strictly positive, \(\tanh\) is also
definable in \(\mathbb{R}_{\exp}\).
\end{proof}

\begin{proof}[{Proof of Proposition~\ref{prop:modern_activations_definable}.(iii)}]
Let
\[
    f_1(x)\eqdef \frac{1}{1+e^{-x}},
    \qquad
    f_2(x)\eqdef \operatorname{Softplus}(x)=\log(1+e^x).
\]
Both \(f_1\) and \(f_2\) are real-analytic on \(\mathbb{R}\). Moreover,
\[
    \frac{d f_1}{dx}(x)
    =
    f_1(x)\bigl(1-f_1(x)\bigr),
\]
and
\[
    \frac{d f_2}{dx}(x)
    =
    \frac{e^x}{1+e^x}
    =
    \frac{1}{1+e^{-x}}
    =
    f_1(x).
\]
Equivalently,
\[
    \frac{d f_1}{dx}
    =
    P_1(f_1),
    \qquad
    P_1(Y_1)\eqdef Y_1-Y_1^2,
\]
and
\[
    \frac{d f_2}{dx}
    =
    P_2(f_1,f_2),
    \qquad
    P_2(Y_1,Y_2)\eqdef Y_1.
\]
The polynomials \(P_1\in\mathbb{R}[Y_1]\) and
\(P_2\in\mathbb{R}[Y_1,Y_2]\) have degrees at most \(2\). Hence
\((f_1,f_2)\) is a Pfaffian chain on \(\mathbb{R}\) of length \(2\) and degree
\(2\).

Finally,
\[
    \operatorname{Softplus}(x)=f_2(x)
\]
is obtained from this chain by the polynomial \(Q(Y_1,Y_2)=Y_2\), of degree
\(1\). Therefore \(\operatorname{Softplus}\) is Pfaffian on \(\mathbb{R}\), of
format at most \((2,2,1)\). In particular, it is definable in the Pfaffian
closure of the real field.

It is also definable in \(\mathbb{R}_{\exp}\): the map \(x\mapsto 1+e^x\) is
\(\mathbb{R}_{\exp}\)-definable and strictly positive, and the logarithm is
definable in \(\mathbb{R}_{\exp}\) as the inverse of the global exponential
function.
\end{proof}

\begin{proof}[{Proof of Proposition~\ref{prop:modern_activations_definable}.(iv)}]
Since each $A_k$ is semi-algebraic and each $p_k$ is polynomial, the graph of
the restriction $p_k|_{A_k}$ is
\[
    \{(x,y)\in\mathbb{R}^{n+m}: x\in A_k,\ y=p_k(x)\},
\]
which is semi-algebraic. Since the sets $A_1,\dots,A_K$ form a finite
partition of $\mathbb{R}^n$, the graph of $f$ is
\[
    \operatorname{Graph}(f)
    =
    \bigcup_{k=1}^K
    \{(x,y)\in\mathbb{R}^{n+m}: x\in A_k,\ y=p_k(x)\}.
\]
This is a finite union of semi-algebraic sets, hence is semi-algebraic.
Therefore $f$ is semi-algebraic, and thus definable in the real field
$(\mathbb{R},+,\cdot,<)$.

For the activations listed above, this framework applies directly. Indeed,
\[
    \operatorname{ReLU}(x)
    =
    \begin{cases}
        0, & x<0,\\
        x, & x\ge 0,
    \end{cases}
\]
and, for fixed \(\alpha\in\mathbb{R}\),
\[
    \operatorname{LeakyReLU}_{\alpha}(x)
    =
    \begin{cases}
        \alpha x, & x<0,\\
        x, & x\ge 0.
    \end{cases}
\]
Both are therefore piecewise-polynomial with respect to the finite
semi-algebraic partition
\[
    \mathbb{R}=(-\infty,0)\cup[0,\infty).
\]
Similarly, this directly applies to the parametric ReLU (PReLU), where \(\alpha\) is learned rather than fixed: the map
\((x,\alpha)\mapsto \operatorname{PReLU}(x,\alpha)\) is semi-algebraic.
\end{proof}

\begin{proof}[{Proof of Proposition~\ref{prop:modern_activations_definable}.(v)}]
Consider the semi-algebraic partition
\[
    \mathbb{R}=(-\infty,0)\cup[0,\infty).
\]
On \([0,\infty)\), the map is \(x\mapsto x\), which is polynomial. On
\((-\infty,0)\), the map is
\[
    x\mapsto \alpha(e^x-1),
\]
which is definable in \(\mathbb{R}_{\exp}\). Hence, for fixed
\(\alpha\in\mathbb{R}\), the graph of \(\operatorname{ELU}_{\alpha}\) is
\[
\begin{aligned}
    \operatorname{Graph}(\operatorname{ELU}_{\alpha})
    =
    &\ \{(x,y)\in\mathbb{R}^2:x\ge 0,\ y=x\} \\
    &\ \cup
    \{(x,y)\in\mathbb{R}^2:x<0,\ y=\alpha(e^x-1)\}.
\end{aligned}
\]
This is a finite union of \(\mathbb{R}_{\exp}\)-definable sets, and is
therefore \(\mathbb{R}_{\exp}\)-definable.

The joint statement follows similarly. Define
\[
    E:\mathbb{R}^2\to\mathbb{R},
    \qquad
    E(x,\alpha)\eqdef \operatorname{ELU}_{\alpha}(x).
\]
Then
\[
\begin{aligned}
    \operatorname{Graph}(E)
    =
    &\ \{(x,\alpha,y)\in\mathbb{R}^3:x\ge 0,\ y=x\} \\
    &\ \cup
    \{(x,\alpha,y)\in\mathbb{R}^3:x<0,\ y=\alpha(e^x-1)\}.
\end{aligned}
\]
Again this is a finite union of \(\mathbb{R}_{\exp}\)-definable sets, so \(E\)
is definable in \(\mathbb{R}_{\exp}\).

Finally, SELU is obtained from ELU by multiplication by a fixed scalar
\(\lambda\). Since definable maps are closed under scalar multiplication,
\(\operatorname{SELU}_{\lambda,\alpha}\) is definable in
\(\mathbb{R}_{\exp}\).
\end{proof}

\begin{proof}[{Proof of Proposition~\ref{prop:modern_activations_definable}.(vi)}]
We first treat the exact GELU. Define, on \(U=\mathbb{R}\),
\[
    f_1(x)\eqdef e^{-x^2/2},
    \qquad
    f_2(x)\eqdef \operatorname{erf}\left(\frac{x}{\sqrt{2}}\right).
\]
Both functions are real-analytic on \(\mathbb{R}\). Moreover,
\[
    \frac{d f_1}{dx}(x)
    =
    -x f_1(x),
\]
and, using
\[
    \frac{d}{dt}\operatorname{erf}(t)
    =
    \frac{2}{\sqrt{\pi}}e^{-t^2},
\]
we obtain
\[
    \frac{d f_2}{dx}(x)
    =
    \frac{2}{\sqrt{\pi}}
    e^{-x^2/2}
    \cdot
    \frac{1}{\sqrt{2}}
    =
    \sqrt{\frac{2}{\pi}}\,f_1(x).
\]
Thus
\[
    \frac{d f_1}{dx}
    =
    P_1(x,f_1),
    \qquad
    P_1(X,Y_1)\eqdef -XY_1,
\]
and
\[
    \frac{d f_2}{dx}
    =
    P_2(x,f_1,f_2),
    \qquad
    P_2(X,Y_1,Y_2)\eqdef \sqrt{\frac{2}{\pi}}\,Y_1.
\]
The polynomials \(P_1\) and \(P_2\) have degrees at most \(2\). Hence
\((f_1,f_2)\) is a Pfaffian chain on \(\mathbb{R}\) of length \(2\) and degree
\(2\). Finally,
\[
    \operatorname{GELU}(x)
    =
    \frac{x}{2}\bigl(1+f_2(x)\bigr)
\]
is obtained from this chain by the polynomial
\[
    Q(X,Y_1,Y_2)\eqdef \frac{X}{2}(1+Y_2),
\]
which has degree \(2\). Therefore the exact GELU is Pfaffian on
\(\mathbb{R}\) of format at most \((2,2,2)\). In particular, it is definable in
the Pfaffian closure of the real field.

We next treat the tanh approximation. Let
\[
    c\eqdef \sqrt{\frac{2}{\pi}},
    \qquad
    a\eqdef 0.044715,
    \qquad
    h(x)\eqdef c(x+ax^3),
\]
and define
\[
    g_1(x)\eqdef \tanh(h(x)).
\]
Then \(g_1\) is real-analytic on \(\mathbb{R}\), and
\[
    \frac{d g_1}{dx}(x)
    =
    h'(x)\bigl(1-\tanh^2(h(x))\bigr)
    =
    c(1+3ax^2)\bigl(1-g_1(x)^2\bigr).
\]
Hence
\[
    \frac{d g_1}{dx}
    =
    R(x,g_1),
    \qquad
    R(X,Y_1)\eqdef c(1+3aX^2)(1-Y_1^2),
\]
where \(R\in\mathbb{R}[X,Y_1]\). Thus \((g_1)\) is a Pfaffian chain on
\(\mathbb{R}\). Moreover,
\[
    \operatorname{GELU}_{\tanh}(x)
    =
    \frac{x}{2}\bigl(1+g_1(x)\bigr),
\]
so \(\operatorname{GELU}_{\tanh}\) is Pfaffian.

Finally, since \(\tanh\) is definable in \(\mathbb{R}_{\exp}\), and
\(\operatorname{GELU}_{\tanh}\) is obtained from \(\tanh\) by polynomial
operations and composition with a polynomial, it is definable in
\(\mathbb{R}_{\exp}\). Hence it is also definable in
\(\mathbb{R}_{\operatorname{an},\exp}\).
\end{proof}

\begin{proof}[{Proof of Proposition~\ref{prop:modern_activations_definable}.(vii)}]
Consider Swish on $U=\mathbb{R}\times\mathbb{R}_{>0}$.  Define
\begin{align*}
    f_1(x,\beta)&=e^{-\beta x},\\
    f_2(x,\beta)&=\frac{1}{1+f_1(x,\beta)}.
\end{align*}
Then
$
    \deriv{f_1}{x}=-\beta f_1
$,
$
    \deriv{f_1}{\beta}=-x f_1
$,
$$
    \deriv{f_2}{x}
    =
        -\frac{\deriv{f_1}{x}}{(1+f_1)^2}
    =
        \beta f_1 f_2^2,
$$
and
$
    \deriv{f_2}{\beta}=x f_1 f_2^2
$.
Thus $(f_1,f_2)$ is a Pfaffian chain of order $2$ and degree $4$.  Since
$
    \operatorname{Swish}(x,\beta)=x f_2(x,\beta)
$,
Swish is Pfaffian of format at most $(2,4,2)$.  Pfaffian functions are definable in the Pfaffian closure of the real field.  Moreover, Swish is obtained from polynomials, $\exp$, addition, multiplication, and division by the strictly positive term $1+e^{-\beta x}$; hence it is also definable in $\mathbb{R}_{\exp}$.
\end{proof}

\begin{proof}[{Proof of Proposition~\ref{prop:modern_activations_definable}.(viii)}]
Consider SwiGLU on $U=\mathbb{R}^2\times\mathbb{R}_{>0}$.  Define
\begin{align*}
    f_1(x_2,\beta) &= e^{-\beta x_2},\\
    f_2(x_2,\beta) &= \frac{1}{1+f_1(x_2,\beta)}.
\end{align*}
The same computation as for Swish gives
\begin{align*}
    \frac{\partial f_1}{\partial x_2} &= -\beta f_1, &
    \frac{\partial f_1}{\partial \beta} &= -x_2 f_1,\\
    \frac{\partial f_2}{\partial x_2} &= \beta f_1 f_2^2, &
    \frac{\partial f_2}{\partial \beta} &= x_2 f_1 f_2^2,
\end{align*}
while the derivatives with respect to $x_1$ vanish.  Thus $(f_1,f_2)$ is again a Pfaffian chain of order $2$ and degree $4$.  Since
$
    \operatorname{SwiGLU}(x_1,x_2,\beta)
    =
        x_1x_2 f_2(x_2,\beta),
$
SwiGLU is Pfaffian of format at most $(2,4,3)$.  Pfaffian functions are definable in the Pfaffian closure of the real field.  Moreover, SwiGLU is obtained from polynomials, $\exp$, addition, multiplication, and division by the strictly positive term $1+e^{-\beta x_2}$; hence it is also definable in $\mathbb{R}_{\exp}$.
\end{proof}

\begin{proof}[{Proof of Proposition~\ref{prop:modern_activations_definable}.(ix)}]
Consider the semi-algebraic partition
\[
    \mathbb{R}=(-\infty,0)\cup[0,\infty).
\]
On \([0,\infty)\), we have \(|x|=x\), and hence
\[
    \operatorname{Softsign}(x)
    =
    \frac{x}{1+x}.
\]
On \((-\infty,0)\), we have \(|x|=-x\), and hence
\[
    \operatorname{Softsign}(x)
    =
    \frac{x}{1-x}.
\]
The denominators \(1+x\) on \([0,\infty)\) and \(1-x\) on \((-\infty,0)\)
are strictly positive. Therefore the graph of \(\operatorname{Softsign}\) is
\[
\begin{aligned}
    \operatorname{Graph}(\operatorname{Softsign})
    =
    &\ \{(x,y)\in\mathbb{R}^2:x\ge 0,\ y(1+x)=x\} \\
    &\ \cup
    \{(x,y)\in\mathbb{R}^2:x<0,\ y(1-x)=x\}.
\end{aligned}
\]
This is a finite union of sets defined by polynomial equalities and
inequalities, hence is semi-algebraic. Consequently, \(\operatorname{Softsign}\)
is definable in the real field \((\mathbb{R},+,\cdot,<)\).
\end{proof}

\begin{proof}[{Proof of Proposition~\ref{prop:modern_activations_definable}.(x)}] The
composition
\[
    x\mapsto \tanh(\operatorname{Softplus}(x))
\]
is definable in \(\mathbb{R}_{\exp}\). Multiplication by \(x\) preserves
definability, so
\[
    \operatorname{Mish}(x)
    =
    x\,\tanh(\operatorname{Softplus}(x))
\]
is definable in \(\mathbb{R}_{\exp}\). Since
\(\mathbb{R}_{\operatorname{an},\exp}\) expands \(\mathbb{R}_{\exp}\), Mish is
also definable in \(\mathbb{R}_{\operatorname{an},\exp}\).
\end{proof}

\begin{proof}[Proof of Proposition~\ref{prop:piecewise_affine_gating}]
For the maxout map, write $\ell_k(x)\eqdef a_k^{\top} x+b_k$.  Its graph is
$$
\operatorname{Graph}(\sigma)
=
\bigcup_{k=1}^K
\Big\{
(x,y,(a_j,b_j)_{j=1}^K):
y=\ell_k(x)
\ \wedge\
\bigwedge_{j=1}^K \ell_k(x)\ge \ell_j(x)
\Big\}.
$$
Each set in the union is cut out by finitely many polynomial equalities and inequalities in $(x,y,(a_j,b_j)_{j=1}^K)$.  Hence the graph is semi-algebraic.

For winner-take-all gating, let $R_1,\dots,R_K$ be the given finite
polyhedral partition of $\mathbb{R}^d$, and suppose that the affine map
$x\mapsto A_kx+b_k$ is applied on $R_k$.  Since each $R_k$ is polyhedral, it is
semi-algebraic.  The graph of the corresponding gating map is
$$
\operatorname{Graph}(\sigma)
=
\bigcup_{k=1}^K
\Big\{
(x,y):
x\in R_k
\ \wedge\
y=A_kx+b_k
\Big\}.
$$
Each set in the union is defined by finitely many polynomial equalities and
inequalities, and hence is semi-algebraic.  Since finite unions of
semi-algebraic sets are semi-algebraic, the graph is semi-algebraic.

Therefore both maps are semi-algebraic, hence definable in the real field
$(\mathbb{R},+,\cdot,<)$ and in every o-minimal expansion of the real field.
\end{proof}

\begin{proposition}[Definability of Softmax]
\label{prop:softmax_definable}
Fix $K\in\mathbb{N}_+$.  The softmax map
\[
    \operatorname{softmax}:\mathbb{R}^K\to\mathbb{R}^K,
    \qquad
    \operatorname{softmax}(z)_k
    \eqdef
    \frac{e^{z_k}}{\sum_{j=1}^K e^{z_j}},
\]
is definable in $\mathbb{R}_{\exp}$.  Consequently, any finite softmax-weighted
aggregation map whose scores and values are definable in $\mathbb{R}_{\exp}$ is
again definable in $\mathbb{R}_{\exp}$.
\end{proposition}

\begin{proof}
Each coordinate $z\mapsto e^{z_k}$ is definable in $\mathbb{R}_{\exp}$, and
the denominator
\[
    \sum_{j=1}^K e^{z_j}
\]
is a finite sum of strictly positive $\mathbb{R}_{\exp}$-definable functions.
Hence the denominator is strictly positive and definable.  Division by a
strictly positive definable function preserves definability, so each coordinate
of $\operatorname{softmax}$ is definable in $\mathbb{R}_{\exp}$.

If the score maps $s_1,\dots,s_K$ and value maps
$v_1,\dots,v_K$ are definable, then
\[
    x\mapsto
    \sum_{k=1}^K
    \frac{e^{s_k(x)}}{\sum_{j=1}^K e^{s_j(x)}}\,v_k(x)
\]
is obtained from definable maps by composition, finite summation,
multiplication, and division by a strictly positive definable denominator.
Thus it is definable in $\mathbb{R}_{\exp}$.
\end{proof}

This directly applies to the attention mechanism.

\begin{proof}[{Proof of Proposition~\ref{prop:multihead_attention_definable}}]
For each index $h,n,m$, the score map
$
    (X,W_Q^h,W_K^h)
    \mapsto
    \lambda\langle W_Q^hX_n,W_K^hX_m\rangle/\sqrt{d_k}
$
is polynomial in the coordinates of $(X,W_Q^h,W_K^h)$.  Since the exponential
map is definable in $\mathbb{R}_{\exp}$, the numerator of each attention weight
is $\mathbb{R}_{\exp}$-definable.  The denominator is a finite sum of strictly
positive $\mathbb{R}_{\exp}$-definable functions, hence is strictly positive
and definable.  Therefore each attention weight
$\alpha^{(h)}_{n,m}$ is definable in $\mathbb{R}_{\exp}$, jointly in
$X,W_Q^h,W_K^h$.

Next, the value map
$
    (X,W_V^h)\mapsto W_V^hX_m
$
is polynomial, hence semi-algebraic.  Therefore each head output
$
    (X,W_Q^h,W_K^h,W_V^h)
    \mapsto
    \operatorname{Attn}_h(X)_n
    =
    \sum_{m=1}^N
    \alpha^{(h)}_{n,m}(X)\,W_V^hX_m
$
is obtained from definable maps by finite summation and multiplication, and is
therefore definable in $\mathbb{R}_{\exp}$.

Finally, concatenation over heads is definable, and the output projection
$
    ((z_1,\dots,z_H),W_O)
    \mapsto
    W_O(z_1,\dots,z_H)
$
is polynomial in the concatenated head outputs and in $W_O$.  Hence
$\operatorname{MHA}$ is definable in $\mathbb{R}_{\exp}$, jointly in
$X$ and the parameters $(W_Q^h,W_K^h,W_V^h)_{h=1}^H,W_O$.
\end{proof}


\begin{corollary}[Definability of Multi-Head Cross-Attention]
\label{cor:multihead_cross_attention_definable}
Fix $N_Q,$ $N_C$, $H$, $d_Q,d_C$, $d_k,d_v,d_{out}\in\mathbb{N}_+$ and $\lambda>0$.  For
each $h\in\{1,\dots,H\}$, let
$W_Q^h\in\mathbb{R}^{d_k\times d_Q}$,
$W_K^h\in\mathbb{R}^{d_k\times d_C}$, and
$W_V^h\in\mathbb{R}^{d_v\times d_C}$, and let
$W_O\in\mathbb{R}^{d_{out}\times Hd_v}$.  For
$X\in\mathbb{R}^{N_Q\times d_Q}$ and
$Y\in\mathbb{R}^{N_C\times d_C}$, define
\[
    \operatorname{CrossAttn}_h(X,Y)_n
    \eqdef
    \sum_{m=1}^{N_C}
    \frac{
        \exp\!\big(\lambda\langle W_Q^hX_n,W_K^hY_m\rangle/\sqrt{d_k}\big)
    }{
        \sum_{\ell=1}^{N_C}
        \exp\!\big(\lambda\langle W_Q^hX_n,W_K^hY_\ell\rangle/\sqrt{d_k}\big)
    }
    W_V^hY_m .
\]
Then the parametrized multi-head cross-attention map
\[
    (X,Y,(W_Q^h,W_K^h,W_V^h)_{h=1}^H,W_O)
    \longmapsto
    \operatorname{CrossMHA}(X,Y),
\]
where
$\operatorname{CrossMHA}(X,Y)_n
\eqdef
W_O(\operatorname{CrossAttn}_1(X,Y)_n,\dots,\operatorname{CrossAttn}_H(X,Y)_n)$,
is definable in $\mathbb{R}_{\exp}$, jointly in $(X,Y)$ and the parameters.
\end{corollary}

\begin{proof}
The proof is identical to Proposition~\ref{prop:multihead_attention_definable}:
the scores are polynomial, the softmax weights are definable in
$\mathbb{R}_{\exp}$, the value maps and output projection are polynomial, and
only finite sums, products, concatenations, and divisions by strictly positive
definable denominators are used.
\end{proof}

\begin{corollary}[Definability of Sliding-Window Self-Attention]
\label{cor:sliding_window_attention_definable}
Fix $N,H,d_{in},d_k,d_v,d_{out}\in\mathbb{N}_+$ and $\lambda>0$.  For each
$n\in\{1,\dots,N\}$, let $\mathcal{W}(n)\subseteq\{1,\dots,N\}$ be a fixed
nonempty finite window.  With the notation of
Proposition~\ref{prop:multihead_attention_definable}, define
\[
    \operatorname{SWAttn}_h(X)_n
    \eqdef
    \sum_{m\in\mathcal{W}(n)}
    \frac{
        \exp\!\big(\lambda\langle W_Q^hX_n,W_K^hX_m\rangle/\sqrt{d_k}\big)
    }{
        \sum_{\ell\in\mathcal{W}(n)}
        \exp\!\big(\lambda\langle W_Q^hX_n,W_K^hX_\ell\rangle/\sqrt{d_k}\big)
    }
    W_V^hX_m .
\]
Then the corresponding parametrized multi-head sliding-window self-attention map
is definable in $\mathbb{R}_{\exp}$, jointly in $X$ and the parameters.
\end{corollary}

\begin{proof}
The proof is identical to Proposition~\ref{prop:multihead_attention_definable},
except that the finite sums are taken over the fixed windows
$\mathcal{W}(n)$ rather than over all tokens.  Fixed finite restrictions,
finite sums, products, exponentials, concatenations, and divisions by strictly
positive definable denominators preserve definability.
\end{proof}

\begin{proof}[{Proof of Proposition~\ref{prop:normalization_layers_definable}}]
For each $G\in\mathcal{G}$, the map $x\mapsto\mu_G(x)$ is affine, and the map $x\mapsto\sigma_G^2(x)$ is polynomial.  Since $\varepsilon>0$, one has
$
    \sigma_G^2(x)+\varepsilon>0
$
for every $x\in\mathbb{R}^d$.  Moreover, the graph of
$
    x\mapsto \sqrt{\sigma_G^2(x)+\varepsilon}
$
is
$$
    \Big\{
        (x,y)\in\mathbb{R}^{d+1}:
        y\ge 0,\,
        y^2=\sigma_G^2(x)+\varepsilon
    \Big\},
$$
which is semi-algebraic.  Hence $x\mapsto \sqrt{\sigma_G^2(x)+\varepsilon}$ is semi-algebraic.

Since the denominator is strictly positive, division by this term is also semi-algebraic.  Thus each coordinate map
$$
    x\mapsto
        \gamma_i
        \frac{x_i-\mu_{G(i)}(x)}
        {\sqrt{\sigma_{G(i)}^2(x)+\varepsilon}}
        +\beta_i
$$
is semi-algebraic.  Therefore $\operatorname{Norm}_{\gamma,\beta,\varepsilon}$ is semi-algebraic.

Layer normalization is obtained by taking a single block $G=\{1,\dots,d\}$.  Group normalization and instance normalization correspond to fixed partitions of the coordinates into groups.  Finite-batch batch normalization is the same construction after viewing the batch and feature coordinates as one finite-dimensional input and choosing the blocks over which the batch statistics are computed.  At inference time, batch normalization uses fixed population statistics, and is therefore affine.  Finally, RMS normalization follows identically, since $x\mapsto |G|^{-1}\sum_{j\in G}x_j^2$ is polynomial.  This proves the claim.
\end{proof}

\begin{proof}[{Proof of Proposition~\ref{prop:embedding_layers_definable}}]
The graph of $\operatorname{Emb}$ is
\[
    \operatorname{Graph}(\operatorname{Emb})
    =
    \bigcup_{r=1}^N
    \Big\{
        (i,E,y)\in
        \mathbb{R}\times\mathbb{R}^{N\times d}\times\mathbb{R}^d
        :
        i=r,\ y=E_{r,:}
    \Big\}.
\]
For each $r\in\{1,\dots,N\}$, the set
\[
    \Big\{
        (i,E,y):
        i=r,\ y=E_{r,:}
    \Big\}
\]
is defined by the polynomial equalities $i-r=0$ and
$y_j-E_{rj}=0$ for $j=1,\dots,d$.  Hence each set in the union is
semi-algebraic.  Since finite unions of semi-algebraic sets are
semi-algebraic, $\operatorname{Graph}(\operatorname{Emb})$ is semi-algebraic.
Therefore $\operatorname{Emb}$ is semi-algebraic, and hence definable in the
real field and in every o-minimal expansion of the real field.
\end{proof}

\begin{proof}[{Proof of Proposition~\ref{prop:bounded_fourier_positional_encoding_definable}}]
For each $j=1,\dots,M$, the phase map
$
    t\mapsto 2\pi\langle \omega_j,t\rangle+\varphi_j
$
is affine, hence semi-algebraic.  Since $D$ is bounded, its image under this affine map is contained in some bounded interval $[-R_j,R_j]$.  The restrictions of $\sin$ and $\cos$ to $[-R_j,R_j]$ are restricted analytic functions, and are therefore definable in $\mathbb{R}_{\mathrm{an}}$.  Hence each coordinate map
$$
    t\mapsto
        \sin\big(2\pi\langle \omega_j,t\rangle+\varphi_j\big),
    \qquad
    t\mapsto
        \cos\big(2\pi\langle \omega_j,t\rangle+\varphi_j\big)
$$
is definable in $\mathbb{R}_{\mathrm{an}}$.  Finite concatenation and affine readouts preserve definability, proving the first claim.

If $D=\{1,\dots,N\}$ is finite, then the graph of $\operatorname{PE}_{\Omega,\varphi}$ is finite.  Hence it is semi-algebraic.
\end{proof}

\begin{remark}[Why boundedness is necessary]
\label{rmk:global_sine_not_definable}
The boundedness assumption is essential.  The global functions $\sin:\mathbb{R}\to\mathbb{R}$ and $\cos:\mathbb{R}\to\mathbb{R}$ are not definable in any o-minimal expansion of the real field, since their zero sets have infinitely many connected components.  Thus Fourier positional encodings are definable only when positions are fixed finite objects, or when the continuous position domain is bounded.  If the frequencies are also treated as trainable parameters, then those frequency parameters must likewise range over a bounded definable set; otherwise the map $(t,\omega)\mapsto\sin(2\pi\langle\omega,t\rangle)$ contains global sine as a slice.
\end{remark}

\begin{proof}[{Proof of Proposition~\ref{prop:pooling_layers_definable}.(i)}]
The map $P_{\operatorname{avg}}$ is affine, hence polynomial.  Therefore its graph is algebraic, and in particular semi-algebraic.  Since polynomial maps are Pfaffian functions, $P_{\operatorname{avg}}$ is Pfaffian.
\end{proof}

\begin{proof}[{Proof of Proposition~\ref{prop:pooling_layers_definable}.(ii)}]
We first treat the case $d=2$.  For $a,b\in\mathbb{R}$,
$$
    \max\{a,b\}
    =
        \frac{a+b+|a-b|}{2}.
$$
The graph of $t\mapsto |t|$ is semi-algebraic, since
$$
    \operatorname{graph}(|\cdot|)
    =
    \big\{
        (t,s)\in\mathbb{R}^2:
        s\ge 0,\,
        s^2=t^2
    \big\}.
$$
Thus $(a,b)\mapsto \max\{a,b\}$ is semi-algebraic.  For general $d\in\mathbb{N}_+$, one writes
$$
    P_{\max}(x_1,\dots,x_d)
    =
        \max\big\{
            x_1,
            \max\{x_2,\dots,x_d\}
        \big\}
$$
and iterates the two-variable construction finitely many times.  Since finite compositions of semi-algebraic maps are semi-algebraic, $P_{\max}$ is semi-algebraic.  Hence it is definable in the real field $(\mathbb{R},+,\cdot,<)$, and therefore in every o-minimal expansion of the real field.
\end{proof}

\begin{proof}[{Proof of Proposition~\ref{prop:pooling_layers_definable}; final claim}]
A pooling layer acting on finitely many fixed windows is obtained by applying $P_{\operatorname{avg}}$ or $P_{\max}$ to each window, followed by coordinate projection and finite concatenation.  Coordinate projections, finite products, and finite concatenations preserve definability.  Therefore any finite pooling layer built from fixed average-pooling and max-pooling windows is definable.
\end{proof}

\begin{proof}[{Proof of Proposition~\ref{prop:deq_layers_definable}}]
Consider the fixed-point set
\[
    \Gamma
    \eqdef
        \big\{
            (x,z)\in X\times Z:
            z=F(x,z)
        \big\}.
\]
Since $F$ is $\mathfrak{S}$-definable, the set $\Gamma$ is
$\mathfrak{S}$-definable.  By the uniqueness assumption, every fiber
\[
    \Gamma_x
    \eqdef
        \big\{
            z\in Z:
            (x,z)\in\Gamma
        \big\}
\]
is a singleton.  Hence $\Gamma$ is precisely the graph of the map
$x\mapsto z_x$.  Therefore $x\mapsto z_x$ is $\mathfrak{S}$-definable.

Since $G$ is $\mathfrak{S}$-definable, the composition
\[
    x\mapsto G(x,z_x)
\]
is also $\mathfrak{S}$-definable.  This is exactly
$\operatorname{DEQ}$, proving the claim.

Finally, affine maps, piecewise-polynomial activations, and normalization
layers are semi-algebraic, while softmax attention is definable in
$\mathbb{R}_{\exp}$.  Finite compositions of $\mathbb{R}_{\exp}$-definable
maps remain $\mathbb{R}_{\exp}$-definable, so any DEQ layer built from such
blocks is definable in $\mathbb{R}_{\exp}$ whenever its equilibrium is
uniquely selected.
\end{proof}

\begin{proof}[{Proof of Corollary~\ref{cor:fixed_mlp_sample_complexity}}]
Applying Proposition~\ref{prop:key_insight} to Propositions~\ref{prop:polynomial_maps_definable} and~\ref{prop:modern_activations_definable__compressed} implies the joint definability of $f_{\theta}$ over all of $\mathbb{R}^{d_0}\times \mathbb{R}^{P}$.  
Thus, Theorems~\ref{thrm:class} and~\ref{thrm:reg} apply.
\end{proof}

\begin{proof}[{Proof of Corollary~\ref{cor:fixed_transformer_sample_complexity}}]
Applying Proposition~\ref{prop:key_insight} to Proposition~\ref{prop:embedding_layers_definable}, \ref{prop:bounded_fourier_positional_encoding_definable}, ~\ref{prop:normalization_layers_definable}, \ref{prop:modern_activations_definable__compressed}, we deduce that $f_{\theta}$ is jointly definable in $(x,\theta)\in \mathbb{R}^{H_1\times d_1}\times \mathbb{R}^P$.  Thus, Theorems~\ref{thrm:class} and~\ref{thrm:reg} apply.
\end{proof}


\end{document}